\documentclass[11pt,a4paper]{article}
\usepackage[T1]{fontenc}
\usepackage{lmodern}
\usepackage[margin=0.83in]{geometry}
\usepackage{amsmath,amssymb,amsthm,mathtools,bm}
\usepackage{graphicx}
\graphicspath{{figures/}}
\usepackage{booktabs,tabularx,array}
\usepackage{enumitem}
\usepackage{microtype}
\usepackage{xcolor}
\usepackage{algorithm}
\usepackage{algpseudocode}
\usepackage{tikz}
\usepackage{pgfplots}
\pgfplotsset{compat=1.18}
\usepgfplotslibrary{groupplots}
\usetikzlibrary{positioning,arrows.meta,patterns}
\usepackage[round,authoryear]{natbib}
\usepackage{hyperref}
\usepackage[nameinlink,capitalise,noabbrev]{cleveref}
\hypersetup{colorlinks=true,linkcolor=blue!50!black,citecolor=blue!50!black,urlcolor=blue!50!black,pdftitle={Certify the Decision, Not Every Score: Exact Top-k Selection from Few Evaluations},pdfauthor={Yuhe Sui; Yingzhi Tang}}
\setlist{nosep,leftmargin=*}
\newtheorem{theorem}{Theorem}

\newtheorem{proposition}[theorem]{Proposition}
\newtheorem{corollary}[theorem]{Corollary}
\newtheorem{definition}[theorem]{Definition}

\newcommand{\ES}{\operatorname{ES}}
\newcommand{\Ccert}{C_{\mathrm{cert}}^\star}
\newcommand{\Risk}{\rho}
\newcommand{\poly}{\operatorname{poly}}
\newcommand{\R}{\mathbb{R}}
\newcommand{\id}{\mathrm{id}}
\newcommand{\cas}[1]{\textnormal{\textsc{CAS-#1}}}
\newcommand{\caspg}{\cas{PG}}
\newcommand{\OPEN}{\textnormal{\textsc{Open}}}
\providecommand{\Description}[1]{}
\algrenewcommand\algorithmicrequire{\textbf{Input:}}
\algrenewcommand\algorithmicensure{\textbf{Output:}}
\definecolor{cascharcoal}{HTML}{252525}
\definecolor{casslate}{HTML}{657681}
\definecolor{casgrey}{HTML}{A9B3BC}
\definecolor{casred}{HTML}{B5161B}
\pgfplotsset{casbars/.style={
    xbar, /pgf/bar shift=0pt, xmin=0, y dir=reverse,
    axis x line*=bottom, axis y line*=left, axis line style={cascharcoal},
    tick style={cascharcoal}, ytick style={draw=none},
    tick label style={font=\scriptsize}, label style={font=\scriptsize},
    title style={font=\footnotesize, yshift=-3pt},
    xmajorgrids, grid style={black!10},
    nodes near coords, point meta=explicit symbolic,
    every node near coord/.append style={font=\scriptsize, anchor=west, text=cascharcoal, inner sep=2pt},
    scaled x ticks=false}}
\title{\textbf{Certify the Decision, Not Every Score:}\\ Exact Top-$k$ Selection from Few Evaluations}
\author{%
\textbf{Yuhe Sui}\\
\small Quantitative Research Society, Singapore\\
\small Nanyang Technological University, Singapore\\
\small \href{https://orcid.org/0009-0002-7456-0804}{ORCID: 0009-0002-7456-0804}
\and
\textbf{Yingzhi Tang}\\
\small Quantitative Research Society, Singapore\\
\small The Chinese University of Hong Kong, Shenzhen, China\\
\small \href{https://orcid.org/0009-0004-7571-0100}{ORCID: 0009-0004-7571-0100}
}
\date{October 2026}
\begin{document}
\maketitle
\begin{abstract}
Choosing a few robust algorithms or schedules can require evaluating every candidate under many adverse scenarios, even though an agent only needs to identify the winners. A predictor may suggest those winners, but cannot rule out a surprise in an unevaluated scenario. We introduce Certified Adaptive Selection (CAS), which returns the exact top-$k$ set as soon as remaining uncertainty cannot change it. CAS keeps guaranteed lower and upper bounds on Expected Shortfall (ES), the average loss in the worst scenarios. Its stopping test compares the worst admissible selected candidate with the best admissible excluded one; for information-disjoint Cartesian intervals this one-cut test is necessary and sufficient. Prediction has a separate role: selecting the next outcome to evaluate, never authorizing the decision. For empirical ES with exact unit-cost queries, we compute in polynomial time the fewest evaluations that could have proved the realized decision in hindsight. Our prediction-guided agent \caspg{} is exact for every predictor, including adversarial ones. On two public algorithm-selection matrices, \caspg{} certifies the exact top three after 630 of 3150 and 981 of 2200 replayed queries, or $1.14$ and $1.18$ times that hindsight minimum. Without historical outcomes, its transcript-only variant requires $1.28$ and $1.37$ times the minimum, outperforming the tested non-predictive certified policies. Separately, a prediction-free CAS policy avoids $93.0\%$ of exact recourse solves in synthetic energy scheduling. One certificate thus supports several acquisition policies and separately proved risk/cost extensions: prediction changes the work, not the justified choice.
\end{abstract}
\medskip

\tableofcontents
\bigskip
\section{Introduction}
\label{sec:intro}

An agent must retain a small set of candidates, but each candidate is judged across many possible adverse scenarios. The objective may be to keep the three algorithms with the smallest worst-case-average loss across benchmark instances, or to choose a few schedules with the lowest tail cost across service conditions. In such tasks, one candidate--scenario evaluation can be a solver run or an exact recourse optimization. A financial hedge screen poses a related question when every market shock requires nonlinear repricing. The number of evaluations grows as the number of candidates times the number of scenarios, even though the agent ultimately needs a set of winners, not a complete matrix of scores.

There is a sharp distinction between \emph{predicting} that a candidate is good and \emph{knowing} that no unevaluated outcome could change the choice. Exhaustive evaluation resolves the second question but is expensive. A learned surrogate may resolve the first cheaply yet fail at the selection boundary. Accurate average predictions do not certify which candidate is ranked $k$ rather than $k+1$.

\textbf{Our answer is to give information acquisition and decision authority separate jobs.} Certified Adaptive Selection (CAS) maintains a valid interval for each candidate's tail-risk score. It stops only after the largest possible score of a selected candidate is smaller than the smallest possible score of every excluded candidate (with a fixed rule for ties). For example, if the prospective winner's risk lies in $[0.20,0.40]$ and its nearest competitor's in $[0.46,0.70]$, further score evaluations cannot change their ordering. A predictor may help the agent reach that separation sooner, but it cannot authorize the stop. Thus the agent certifies \emph{membership}, not every individual score.

This separation makes efficiency measurable. On a fully specified realization we can ask for the smallest collection of outcomes whose revelation already forces the right top-$k$ set. We call its cost the \emph{hindsight certificate optimum}, $\Ccert$. It is not an online policy: its query set may exploit information hidden from the agent. Our one-cut characterization is necessary and sufficient when compatible losses of different candidates can vary jointly; for empirical Expected Shortfall (ES), we show how to compute $\Ccert$ and an attaining query set in polynomial time. This supplies an exact instance-specific yardstick, rather than only comparison with evaluating every candidate--scenario outcome.

The online policy \caspg{} then asks which observation should be bought next. It works on the candidates currently blocking the certificate and prioritizes scenarios predicted to contribute to the tail. Predictions may come from a historical model, from values already revealed during the run, or from both. The proof of exactness is unchanged if these predictions are unhelpful: the policy always checks the same deterministic certificate before returning a decision. A worst-case guarantee of being close to $\Ccert$ is impossible for zero-error online acquisition without further information; the empirical question is whether prediction closes the observed gap.

On two public ASlib matrices, \caspg{} takes 630 of 3150 and 981 of 2200 exact-value queries, respectively: $1.14\times$ and $1.18\times\Ccert$. A like-for-like variant using \emph{no historical outcomes} needs 709 and 1144, versus at least 1872 and 1699 for the tested deterministic non-predictive certified schedulers. These are query counts in a replay oracle, not measured reductions in algorithm-running time. A separate synthetic energy study tests genuine recourse solves with a non-predictive CAS scheduler; it avoids $93.0\%$ of the full solve count. Neither result transfers the observed public-benchmark cost ratios to shared-primitive financial pricing or to interval-valued oracles.

\paragraph{Contributions and scope.}
The central contribution is \emph{exact risk-sensitive selection with prediction restricted to acquisition}. Its theoretical foundation is a necessary-and-sufficient decision certificate and a computable empirical-ES hindsight benchmark (\cref{sec:model,sec:hindsight}). Its algorithmic realization is \caspg{}, whose exactness is independent of its predictor and whose query cost is assessed against that benchmark (\cref{sec:caspg,sec:public}). Native recourse solves and controlled stress tasks test what the certificate layer achieves under other conditions (\cref{sec:energy,sec:stress}). We separately state extensions to general spectral risk and priced certificates, and the distinct assumptions of the motivating financial setting (\cref{sec:spectral}). The complete proofs, protocol specifications, diagnostic figures, and evidence boundaries are provided in the appendices, allowing this version to be read and checked without a separate supplement.

\paragraph{Guide to the extended version.}
The main text establishes the decision problem, the certificate, the hindsight optimum, the online policy, and the experiments in that order. \Cref{supp:cut,supp:es,supp:caspg,supp:online} provide technical verification of the central results; \cref{supp:coarsening,supp:lower,supp:structural,supp:ptas} contain the general-risk and priced-certificate theory; and \cref{supp:protocol,supp:merton} document the application-specific evidence and the nonlinear pricing enclosure. These appendices are part of this paper, not separate claims of production FRTB speedup.
\section{When is the selected set already certain?}
\label{sec:model}

\subsection{From intervals to a decision}

Consider selecting the single lowest-risk candidate from three candidates. Suppose their current \emph{certified} risk ranges are $[0.20,0.40]$, $[0.46,0.70]$ and $[0.80,0.91]$. The first candidate is already certain to win, although none of its scenario losses need be known exactly. By contrast, if its upper bound were $0.55$, the first two ranges would overlap. Their current bounds alone would not settle which candidate wins. This is the distinction on which CAS acts: uncertainty about scores is acceptable until it can change the chosen set.

In general, let $m$ candidates have $S$ losses each. The loss vector of candidate $j$ is $x_j=(x_{j1},\ldots,x_{jS})$. Before evaluation, each loss lies in a certified public interval, giving the Cartesian box
$B_j^0=\prod_{s=1}^S[\ell^0_{js},u^0_{js}]$.
A \emph{query} of $(j,s)$ returns $x_{js}$ exactly and replaces that coordinate's interval by a point. Unless stated otherwise, each query costs one unit. The theorem model is \emph{information-disjoint}: rows share no queryable primitive, and every combination of compatible row values is feasible. This is a property of the uncertainty set, not a claim of stochastic independence.

We compare candidates by a tail-risk score: larger adverse losses count more. Empirical Expected Shortfall with tail mass $t\in(0,S]$ averages the worst $t$ scenario losses, with fractional weight on the boundary loss when $t$ is noninteger \citep{AcerbiTasche2002,RockafellarUryasev2002}. More generally, for descending losses $y_{(1)}\ge\cdots\ge y_{(S)}$, a coherent spectral score is \citep{Artzner1999,Acerbi2002}
\begin{equation}
\Risk_\lambda(y)=\sum_{r=1}^S\lambda_r y_{(r)},\qquad
\lambda_1\ge\cdots\ge\lambda_S\ge0,\quad \sum_r\lambda_r=1.
\label{eq:spectral}
\end{equation}
ES is the particular spectrum that gives weight $1/t$ to the worst $t$ losses (and the appropriate fractional boundary weight). Smaller scores are better. Fixed distinct candidate IDs break ties: $(\Risk_\lambda(x_j),\id_j)$ is compared lexicographically, denoted $\prec$, and $D_k(x)$ is the set of the $k$ smallest such keys.

\begin{definition}[Decision certificate]
\label{def:certificate}
For a set $Q$ of evaluated candidate--scenario primitives, let $B_Q(x)$ be the original Cartesian box restricted to completions agreeing with the realized losses $x$ on $Q$. The query set $Q$ \emph{certifies} the decision when $D_k(v)=D_k(x)$ for every $v\in B_Q(x)$.
\end{definition}

The definition requires invariance over \emph{all} compatible completions. No error probability or assumed predictor calibration is involved. Because spectral risk is coordinatewise monotone, the smallest and largest compatible risk scores of row $j$ are attained at its current lower and upper corners. After queries $Q_j$ in that row, write these exact bounds as
$L_j(Q_j)=\Risk_\lambda(\ell_j^0[Q_j\leftarrow x_j])$ and
$U_j(Q_j)=\Risk_\lambda(u_j^0[Q_j\leftarrow x_j])$.

\subsection{A necessary and sufficient stopping test}

The illustrated non-overlap condition extends to any $k$: each selected candidate's \emph{worst} compatible key must precede each excluded candidate's \emph{best} compatible key. Under the Cartesian model this is not just sufficient, but necessary.

\begin{theorem}[One-cut exact composition]
\label{thm:cut}
Assume Cartesian, information-disjoint rows with additive query costs, attained exact extrema and fixed ID ties. Write $A=D_k(x)$. An allocation certifies $A$ exactly when
\[
\max_{i\in A}(U_i,\id_i)\prec
\min_{j\notin A}(L_j,\id_j).
\]
Equivalently, one lexicographic cut separates the upper keys of all selected rows from the lower keys of all excluded rows. For unit costs, polynomial row frontiers compose by scanning such cuts and summing the least row budgets that cross each cut.
\end{theorem}

\begin{proof}
If the cut holds, every completion gives each selected candidate a key no larger than its upper key, and each excluded candidate a key no smaller than its lower key. The set $A$ is therefore the top-$k$ set in every completion.

For necessity, suppose the cut fails. Then some $i\in A$ and $j\notin A$ have $(U_i,\id_i)\not\prec(L_j,\id_j)$. The row-upper and row-lower corners attain these endpoints. Information-disjointness permits a joint completion realizing the upper corner of $i$ and the lower corner of $j$. There, $j$ precedes $i$ under the fixed ID tie rule, contradicting invariance of $A$.

For the unit-cost composition statement, let $U_i^\star(q)$ and $L_j^\star(q)$ be the attained row frontiers defined in \cref{sec:hindsight}. Any certificate can replace each row query set by an attaining frontier set at no larger cardinality, weakly improving the separating inequalities. For every proposed selected upper-key cut, the cheapest row budgets can therefore be chosen independently; enumerating the finitely many selected frontier keys and summing these budgets yields an optimal original-coordinate certificate.
\end{proof}

Thus the cut is both a precise online stopping rule and the basis for the offline minimum-information calculation. Without information-disjointness, the necessity argument using jointly realizable corners need not apply.

\paragraph{One certificate, several branches.}
The five stages in \cref{fig:model} separate assumptions from agent choices. In the core loop, certified bounds and the one-cut test determine \emph{whether} to stop; the S3 policy determines \emph{what} to buy next. \caspg{}, \cas{PG-T}, \cas{PG-H}, \cas{PG-U}, \cas{PD}, \cas{W}, \cas{F} and \cas{CI} are online query policies. By contrast, \cas{SC}, \cas{SV} and \cas{PT} provide risk-side bounds, structural verification or hindsight price optimization under separate assumptions; they are not alternative online acquisition policies. Native energy changes the evaluation oracle; financial pricing changes both the oracle and the information model. Results from these branches are not interchangeable.
\begin{figure}[tbp]
\centering
\resizebox{0.98\textwidth}{!}{%
\begin{tikzpicture}[font=\scriptsize, >={Stealth[length=4pt]},
  core/.style={draw=cascharcoal, rounded corners=2pt, minimum height=0.91cm,
     text width=2.73cm, align=center, inner sep=2.6pt, fill=black!4},
  cert/.style={core, fill=black!10},
  order/.style={core, draw=casred, line width=0.9pt, fill=casred!3},
  branch/.style={draw=casslate, rounded corners=2pt, minimum height=1.18cm,
    text width=2.73cm, align=center, inner sep=2.4pt, fill=white},
  theory/.style={branch, dashed, fill=black!2},
  policies/.style={branch, draw=casred, fill=casred!3},
  noncore/.style={branch, densely dotted, fill=black!2},
  flow/.style={->, cascharcoal, line width=0.55pt},
  attach/.style={-, casslate, densely dashed, line width=0.55pt}]
\node[core] (s0) at (0,0) {\textbf{S0\; Information}\\ public loss boxes;\\ query and cost model};
\node[cert] (s1) at (3.48,0) {\textbf{S1\; Exact bounds}\\ compute lower / upper\\ ES endpoints};
\node[cert] (s2) at (6.96,0) {\textbf{S2\; Certificate}\\ selected upper $\prec$\\ excluded lower?};
\node[order] (s3) at (10.44,0) {\textbf{S3\; Acquisition}\\ choose unresolved\\ row and scenario};
\node[core] (s4) at (13.92,0) {\textbf{S4\; Evaluation}\\ buy response; update\\ public intervals};
\draw[flow] (s0)--(s1);
\draw[flow] (s1)--(s2);
\draw[flow] (s2)--node[above,inner sep=1pt]{no}(s3);
\draw[flow] (s3)--(s4);
\draw[flow] (s4.north) -- ++(0,0.41) -| (s1.north);
\draw[flow,casred] (s2.south) -- ++(0,-0.32) node[below,align=center,inner sep=1pt,text=cascharcoal]{\textbf{yes: exact top-$k$}};
\node[theory] (b0) at (0,-2.05) {\textbf{\cas{PT}: few price classes}\\ Offline certificate-cost PTAS\\ (separate theory)};
\node[theory] (b1) at (3.48,-2.05) {\textbf{\cas{SC}: coarsen spectrum}\\ \textbf{\cas{SV}: verify structure}\\ Spectral-risk theory};
\node[branch,fill=black!7] (b2) at (6.96,-2.05) {\textbf{One-cut stopping authority}\\ Exact for every query ordering\\ $\Ccert$: offline benchmark};
\node[policies] (b3) at (10.44,-2.05) {\textbf{\cas{PG}, \cas{PG-T}}\\ \textbf{\cas{PG-H}, \cas{PG-U}}\\ \cas{PD}, \cas{W}, \cas{F}, \cas{CI}\\ \emph{First two lines use prediction}};
\node[noncore] (b4) at (13.92,-2.05) {\textbf{ASlib: value replay}\\ \textbf{Energy: exact recourse}\\ Finance: intervals + shared};
\foreach \a/\b in {s0/b0,s1/b1,s2/b2,s3/b3,s4/b4}{\draw[attach] (\a.south) -- (\b.north);}
\node[anchor=west, text=casslate, font=\scriptsize] at (-1.48,-3.05)
 {Core: disjoint rows, exact point queries, unit costs. Dashed: theory; red: query policies; dotted: oracle settings.};
\end{tikzpicture}
}
\Description{Five-stage CAS certificate-and-query loop with branches aligned below the stages they affect. S0 carries the offline priced-query PTAS CAS-PT; S1 carries spectral-coarsening CAS-SC and structural-verification CAS-SV; S2 is the common necessary-and-sufficient one-cut certificate; S3 carries online prediction-guided CAS-PG with transcript-only, history-only and alternative-row variants, and non-predictive CAS-PD, CAS-W, CAS-F and CAS-CI; S4 shows ASlib replay, native exact energy recourse and a distinct financial interval/shared-primitive setting. Financial pricing also changes S0 and lies outside the disjoint exact-point-query theory. Only S2 authorizes termination.}
\caption{CAS is a \emph{shared certificate loop with distinct branches}, not a collection of independently certified predictors. The red S3 branch contains the main agent \caspg{} and alternative query policies; dashed branches handle risk and query-price theory in S0--S1. All core policies use the same S2 stopping rule. Finance changes both S0 and S4, falls outside the core model and has separately established correctness.}
\label{fig:model}
\end{figure}
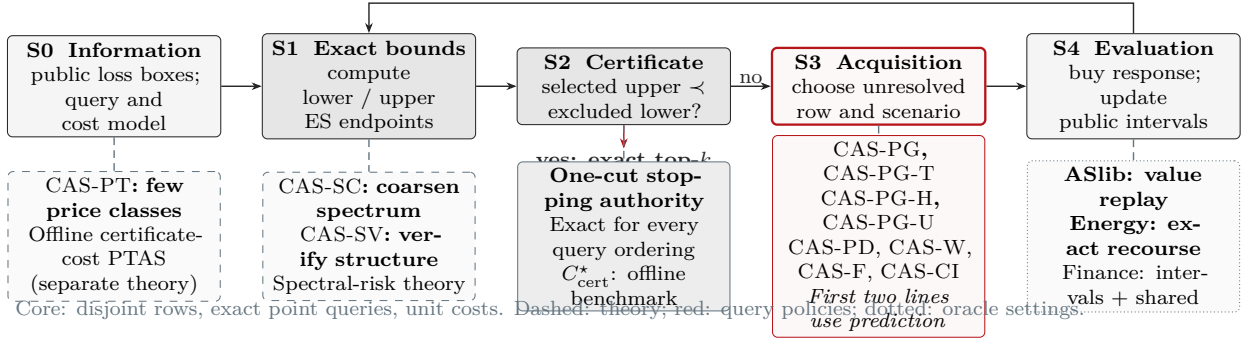
\section{The least information needed in hindsight}
\label{sec:hindsight}

\subsection{An exact denominator for query cost}

A policy can be correct yet purchase far more observations than necessary for a particular realized matrix. To distinguish useful acquisitions from wasted ones, we compare its query count $C$ with the cheapest \emph{certifying} query set on that same matrix. Using the certificate of \cref{def:certificate}, define
\begin{equation}
\Ccert(x)=\min_Q\big\{c(Q): Q\text{ certifies }D_k(x)\big\}.
\label{eq:ccert}
\end{equation}
The minimum knows the realized losses before choosing queries. It is therefore a hindsight information benchmark, not an implementable online lower-bound policy. Every zero-error agent that terminates with a certificate must pay at least $\Ccert(x)$; the ratio $C/\Ccert$ measures its excess acquisition cost relative to this instance-specific baseline.

For candidate $j$, the \emph{upper frontier} is the smallest upper score obtainable by revealing at most $q$ realized losses; the \emph{lower frontier} is the largest lower score obtainable with that budget:
\[
\begin{aligned}
U_j^\star(q)&=\min_{|Q_j|\le q}U_j(Q_j),\\
L_j^\star(q)&=\max_{|Q_j|\le q}L_j(Q_j).
\end{aligned}
\]
These offline constructions use the realized values. \Cref{thm:cut} means that once these frontiers are known, the global optimum can be found by checking separating cuts and adding the smallest budget needed in each row. The computation optimizes evidence for membership, rather than the precision of all risk scores. For a proposed cut, each selected candidate must spend just enough queries to move its upper bound below that cut; each excluded candidate must spend just enough to lift its lower bound above it. Since different rows share no primitive, these row budgets can be added. The global optimum chooses the cheapest cut and combines the attaining row query sets.

\subsection{Why the ES frontiers are tractable}

For empirical ES, the useful variational form is
\begin{equation}
\ES_t(y)=\min_z\Big\{z+\frac1t\sum_s(y_s-z)_+\Big\}.
\label{eq:esdual}
\end{equation}
Consider the upper frontier. For a fixed threshold $z$, revealing coordinate $s$ reduces the positive-part sum by
$g_s(z)=(u_s-z)_+-(x_s-z)_+$.
The best $q$ queries at that threshold therefore have the $q$ largest gains. Their order changes only at $O(S^2)$ rational breakpoints, which can be enumerated. For the lower frontier, choosing revelations can change which observations enter the tail. A rank-selection dynamic program handles that switching; when $t$ is fractional, it also enumerates the coordinate receiving the boundary weight. Full constructions and bit-complexity arguments appear in \cref{supp:es}.

\begin{corollary}[Exact empirical-ES verification]
\label{cor:es}
For rational information-disjoint empirical-ES rows with unit exact-query costs, the complete upper and lower row frontiers, an attaining query set for every budget, and the hindsight optimum $\Ccert$ with an attaining certificate are computable in polynomial bit complexity, including instances in which revelation changes tail membership.
\end{corollary}

The distinction between upper and lower frontiers matters. Revealing a large loss raises a lower bound, helping reject a candidate; revealing a small loss lowers an upper bound, helping retain one. ES can switch which scenarios occupy its worst tail after a revelation, so neither frontier is obtained safely by freezing the original tail. The exact algorithms account for this switching.

The corollary makes a nontrivial empirical comparison possible: the denominator is a \emph{computed minimum}, not a heuristic lower bound or the cost of full evaluation. It does not imply that a policy can identify the optimal queries without knowing the hidden values.

\subsection{Why hindsight cannot be matched uniformly online}

\begin{proposition}[Online boundary]
\label{prop:online}
There are instances with $n$ unresolved primitives and $\Ccert=1$ on which every deterministic zero-error online policy makes $n$ evaluations on some realization, and every randomized zero-error policy incurs expected cost at least $(n+1)/2$ on some realization.
\end{proposition}

The obstruction is a one-hot search problem: one hidden coordinate determines the comparison with a known row, and the public intervals do not distinguish its position. A hindsight verifier queries that coordinate immediately, but a zero-error online agent can find it last. This is the familiar closed-interval search phenomenon in explorable uncertainty \citep{ErlebachHoffmann2015}. It explains why we do not claim a universal competitive ratio against $\Ccert$. Instead, the purpose of prediction is to guide acquisition on actual instances without weakening exactness.
\section{CAS-PG: prediction buys information, not authority}
\label{sec:caspg}

CAS turns the exact stopping criterion into an agent by adding an acquisition policy. At every iteration, the agent computes the current risk bounds and tentatively selects the $k$ candidates with the smallest \emph{lower} keys. Two rows witness the unresolved boundary: the selected row with the largest upper key, and the excluded row with the smallest lower key. If these two keys are ordered correctly, the whole set is certified. Otherwise the agent needs another exact observation.

\Cref{fig:model} also distinguishes online query policies from separate risk/cost theory branches. The certificate evaluates actual returned values and valid intervals. The \emph{query policy} may use arbitrary predictions to decide which observation to buy; its provisional predictions are never substituted for oracle responses.

\subsection{Which observation comes next?}

\caspg{} uses a fixed row rule and a predicted within-row ordering (\cref{alg:caspg}). It first works on the selected boundary row while that row has unresolved coordinates; otherwise it works on the excluded boundary row. Within that row, it imagines revealing each unresolved loss at its predicted value and scores the resulting increase in the row's lower ES bound. It purchases the coordinate with the largest such increase. The pseudo-revelations are only scores for acquisition, not certified updates. On a common $[0,1]$ box, this ranks high predicted losses first: those are most likely to enter the ES tail. The upper ES endpoint often cannot decrease after one revelation while at least $t$ upper-bound coordinates remain, which motivates using the lower-endpoint effect.

An optional prior $p^0_{js}$ predicts losses from historical instance features. The agent can also adjust that prior using the evaluations already in its transcript. If $R_s$ is the set of rows queried in scenario $s$, the correction is
$e_s=\sum_{j'\in R_s}(x_{j's}-p^0_{j's})/(|R_s|+1)$,
with a fixed shrinkage constant of one; the prediction used for ranking is
$\hat p_{js}=\operatorname{clip}(p^0_{js}+e_s,[\ell_{js},u_{js}])$.
The correction shares evidence that scenario $s$ was unexpectedly costly for other candidates. The default \caspg{} uses both sources. \cas{PG-T} uses the interval midpoint as its prior and learns only from the transcript; \cas{PG-H} uses history but no transcript correction. The transcript-only variant is essential for comparing policies that begin with the same public information.

\subsection{Exactness does not depend on predictive quality}

\begin{proposition}[Prediction-independent exactness]
\label{prop:sound}
In the model of \cref{thm:cut} with empirical-ES rows, let the prior $p^0$ be arbitrary (any function of public information, possibly adversarial). For every realization $x$ in the public box, \caspg{}
(i) terminates after at most $mS$ evaluations;
(ii) returns $D_k(x)$;
(iii) stops at the first prefix of its own query sequence that certifies the decision.
Consequently $\Ccert(x)\le C_{\caspg}(x)\le mS$.
\end{proposition}

\begin{proof}
At every iteration the current box contains the realized matrix $x$: a coordinate is replaced by a singleton only after the exact oracle reveals it. The invariant does not depend on the prior or the transcript-based prediction.

If the stopping test passes for the tentative set $A$ of $k$ smallest lower keys, then for every compatible completion $v$, each selected risk key lies below its upper key, and each excluded risk key lies above its lower key. The one-cut test therefore gives $D_k(v)=A$, in particular $A=D_k(x)$.

Conversely, suppose the current transcript already certifies the realized membership set $A'=D_k(x)$. By \cref{thm:cut}, all selected upper keys precede all excluded lower keys. Selected lower keys are at most their upper keys, so the $k$ smallest current lower keys are exactly $A'$. The algorithm consequently chooses $A=A'$ and its stopping test passes. As the test precedes each query, this is the first certifying prefix of the chosen query sequence.

Finally, if the stopping test fails, both boundary rows cannot already be fully evaluated. Otherwise their exact lower and upper keys would respect the strict lexicographic order separating the tentative selected set from its complement, so the test would pass. The selected boundary row, if unresolved, or otherwise the excluded boundary row therefore has a new cell to query. Each iteration buys a distinct cell and there are only $mS$ cells. Termination, exactness, earliest-prefix stopping, and the defining lower bound $\Ccert(x)\le C_{\caspg}(x)$ follow.
\end{proof}

The guarantee is \emph{about the decision}, not the query budget. The interval $[\Ccert,mS]$ is the general cost bound; \cref{prop:online} rules out an online bound close to hindsight on every instance. Predictive quality enters only the empirical question of acquisition efficiency.
\begin{algorithm}[t]
\caption{\caspg{} for ES top-$k$ under exact unit-cost evaluation}
\label{alg:caspg}
\begin{algorithmic}[1]
\Require public boxes $[\ell_{js},u_{js}]$; $k$; tail mass $t$; evaluation oracle; optional prior $p^0_{js}$ (default: interval midpoint)
\Ensure the exact top-$k$ set $D_k(x)$
\Loop
  \State compute exact endpoints $L_j,U_j$ of every row (rational arithmetic)
  \State $A\gets$ the $k$ smallest keys $(L_j,\id_j)$
  \State $i^\star\gets\arg\max_{i\in A}(U_i,\id_i)$;\quad $j^\star\gets\arg\min_{j\notin A}(L_j,\id_j)$
  \State \textbf{if} $(U_{i^\star},\id_{i^\star})\prec(L_{j^\star},\id_{j^\star})$ \textbf{then return} $A$
  \State $r\gets i^\star$ if row $i^\star$ has an unevaluated cell, else $j^\star$
  \State $\hat p_{rs}\gets\operatorname{clip}\big(p^0_{rs}+e_s,\,[\ell_{rs},u_{rs}]\big)$ for unevaluated $s$
  \State $s^\star\gets\arg\max_s\big[L_r(\ell_{rs}\!\leftarrow\!\hat p_{rs})-L_r\big]$ (ties: larger $\hat p_{rs}$, then ID)
  \State evaluate $(r,s^\star)$; set $\ell_{rs^\star}=u_{rs^\star}=x_{rs^\star}$
  \State update the scenario effect $e_{s^\star}$ from the residual $x_{rs^\star}-p^0_{rs^\star}$
\EndLoop
\end{algorithmic}
\end{algorithm}
\section{Public benchmarks: exact decisions from fewer queries}
\label{sec:public}

\begin{figure}[tbp]
\centering
\resizebox{0.98\textwidth}{!}{%
\begin{tikzpicture}
\begin{groupplot}[
  group style={group size=2 by 1, horizontal sep=0.9cm, y descriptions at=edge left},
  casbars, /pgf/bar width=5.4pt, width=0.47\textwidth, height=4.25cm,
  ymin=-0.65, ymax=8.65, ytick={0,...,8},
  yticklabels={Full exact ($mS$), Uniform over all cells, \cas{W}, \cas{F}, Uniform in witness, \cas{PD}, \cas{PG-T} (no history), \caspg{}, Hindsight $\Ccert$},
]
\nextgroupplot[title={OPENML-WEKA-2017 ($30\times105$)}, xmax=4350, xtick={0,1000,2000,3000}]
\addplot[fill=cascharcoal, draw=cascharcoal] coordinates {(3150,0) [3150 ($5.69\times$)]};
\addplot[fill=casgrey, draw=casgrey] coordinates {(3144,1) [3144 ($5.68\times$)] (3074,2) [3074 ($5.55\times$)] (3039,3) [3039 ($5.49\times$)] (2006,4) [2006 ($3.62\times$)] (1872,5) [1872 ($3.38\times$)]};
\addplot[fill=casred, draw=casred] coordinates {(709,6) [\textbf{709 ($\mathbf{1.28\times}$)}] (630,7) [\textbf{630 ($\mathbf{1.14\times}$)}]};
\addplot[pattern=north east lines, pattern color=cascharcoal, draw=cascharcoal] coordinates {(554,8) [554 ($1.00\times$)]};
\draw[cascharcoal, densely dotted] (axis cs:554,-0.65) -- (axis cs:554,8.65);
\nextgroupplot[title={CSP-Minizinc-Obj-2016 ($22\times100$)}, xmax=3050, xtick={0,500,1000,1500,2000}]
\addplot[fill=cascharcoal, draw=cascharcoal] coordinates {(2200,0) [2200 ($2.64\times$)]};
\addplot[fill=casgrey, draw=casgrey] coordinates {(2169,1) [2169 ($2.60\times$)] (1966,2) [1966 ($2.36\times$)] (2191,3) [2191 ($2.63\times$)] (1661.5,4) [1661.5 ($1.99\times$)] (1699,5) [1699 ($2.04\times$)]};
\addplot[fill=casred, draw=casred] coordinates {(1144,6) [\textbf{1144 ($\mathbf{1.37\times}$)}] (981,7) [\textbf{981 ($\mathbf{1.18\times}$)}]};
\addplot[pattern=north east lines, pattern color=cascharcoal, draw=cascharcoal] coordinates {(834,8) [834 ($1.00\times$)]};
\draw[cascharcoal, densely dotted] (axis cs:834,-0.65) -- (axis cs:834,8.65);
\end{groupplot}
\end{tikzpicture}
}
\Description{Two horizontal bar charts, one per public benchmark, of the number of exact evaluations each certified method buys before the decision certificate closes. On OPENML-WEKA-2017: full exact 3150, uniform over all cells 3144, CAS-W 3074, CAS-F 3039, uniform in witness 2006, CAS-PD 1872, CAS-PG-T 709, CAS-PG 630, hindsight optimum 554. On CSP-Minizinc-Obj-2016: 2200, 2169, 1966, 2191, 1661.5, 1699, 1144, 981 and 834.}
\caption{Exact evaluations bought before the decision certificate closes (ES top-3); in parentheses, the multiple of the hindsight optimum $\Ccert$ (dotted line). Every method returns the exact full-information set, so the bars differ only through the order of purchases. Red bars: \cas{PG-T} uses transcript information only, whereas \caspg{} additionally uses historical outcomes from other instances. Grey bars use no predictor. Uniform baselines are medians over 30 seeds.}
\label{fig:public-cost}
\end{figure}
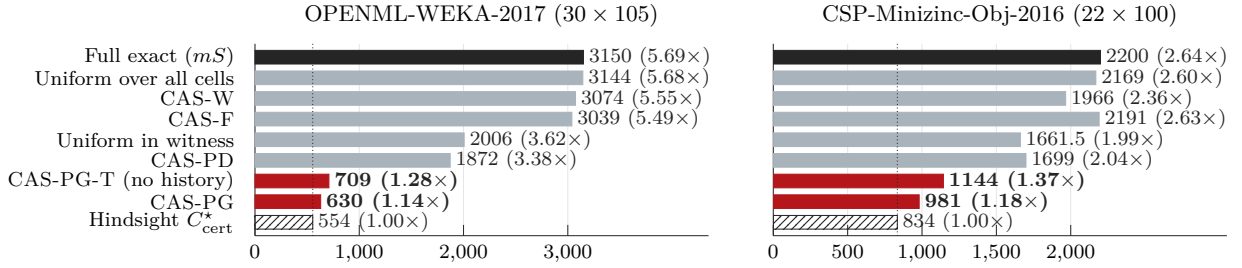

\subsection{What the experiments measure}

We evaluate one fixed risk-sensitive selection task on each of two public ASlib algorithm-selection matrices \citep{BischlEtAl2016}. A \emph{candidate} is an algorithm, a \emph{scenario} is a benchmark instance, and a cell contains published performance of that algorithm on that instance. The agent must identify the three algorithms with smallest ES \emph{across instances}. This differs from ordinary algorithm selection, which chooses one algorithm separately for each instance.

The complete matrices sit behind a hidden-value replay oracle. Revealing a cell counts as one exact algorithm--instance evaluation, but does not execute a new algorithm run. Thus our measured costs are \emph{exact-value query counts}, not observed savings in training or solver wall time. We use OPENML-WEKA-2017 (30 classifiers, 105 datasets, loss $1-$accuracy, ES level $0.10$, tail mass $10.5$) and CSP-Minizinc-Obj-2016 (22 solvers, 100 instances, published normalized objective, timeouts as 1, ES level $0.50$, tail mass 50). In both, $k=3$, the public unevaluated boxes are $[0,1]$, query costs are one, and there is no decision-cut tie. Published decimal strings are parsed into exact rationals; floating point is used only for deciding the query order.

\paragraph{Equal-information comparison.}
Our primary baseline comparison uses \cas{PG-T}, which has only the public boxes and its own query transcript. The non-predictive policies have exactly those sources too. They share the same certificate but select queries differently: \cas{W} targets the widest row bound, \cas{F} uses frontier distance, and \cas{PD} raises dual weight on a universal witness before buying a tight primitive. Random controls query either any unresolved cell or a cell within the current frontier witness. An additional control uses CAS-PG's row rule but orders cells uniformly at random within that row.

\paragraph{Historical-data comparison.}
The full \caspg{} also uses a regression prior trained on public instance meta-features and complete outcomes from \emph{other} instance folds. For each ASlib ten-fold holdout, ridge regression or a 300-tree random forest is selected using inner cross-validation on the training folds, with no outcome from the held-out instance. The historical outcome collection and the one-off predictor fit are uncharged by the replay. This is extra side information, so its improvement is reported separately rather than treated as an equal-information gain.

The task matrices, policy family, comparison set and acceptance thresholds ($C\le1.50\,\Ccert$ on OPENML and $C\le1.75\,\Ccert$ on MiniZinc) were recorded in a protocol file before the first CAS-PG run. The policy family had been developed after pilot values of $\Ccert$ on the matrices were known, with no policy parameter tuned on them; file times alone cannot independently establish the claimed pre-run ordering. All non-random policy runs are deterministic. The full protocol and audit qualifications appear in \cref{supp:caspg}.

\subsection{How much does ordering save?}

The offline minima are $\Ccert=554$ cells on OPENML and 834 on MiniZinc, respectively $17.6\%$ and $37.9\%$ of their full matrices. \cas{PG-T} buys 709 and 1144 cells ($1.28\times$ and $1.37\times\Ccert$). By comparison, the best tested deterministic non-predictive certified schedulers buy 1872 and 1699 cells. The randomized within-row control has medians of 1905.5 and 1641.5 over 30 seeds. \emph{This is the equal-information result}: information learned during acquisition orders subsequent queries sufficiently well to approach the hindsight minimum on both matrices.

Adding the cross-fitted historical prior reduces the default \caspg{} counts to \textbf{630 of 3150} on OPENML and \textbf{981 of 2200} on MiniZinc (\cref{fig:public-cost,tab:public}). These are $20.0\%$ and $44.6\%$ of exhaustive evaluation, or $1.14\times$ and $1.18\times\Ccert$. Both satisfy the pre-declared thresholds. \caspg{} has the lowest exact-query cost among the methods compared in Table 1 on the MiniZinc task; on OPENML, the pre-declared variants \cas{PG-H} (623) and \cas{PG-U} (626) each use slightly fewer queries than the default (630). The main effect is that learning influences \emph{when the certificate can close}, not which candidates it accepts.

\subsection{What explains the remaining gap?}

Replacing the within-row prediction by a uniform order raises the median query count to 1905.5 on OPENML and 1641.5 on MiniZinc, close to the uniform-in-witness control. The benefit on these matrices is therefore driven by which scenarios are evaluated within a blocking row. With history alone, \cas{PG-H} takes 623 and 1043 queries; the transcript correction reduces MiniZinc to 981 but increases OPENML from 623 to 630. This separates a large benefit from predictive ordering from the mixed incremental benefit of combining two prediction sources.

The queries not saved also have an identifiable structure (\cref{tab:gap}). On OPENML, the three selected rows cost 315 queries versus 314 in hindsight; much of the remaining gap comes from the seventh-ranked excluded row, where the policy spends 69 queries against the hindsight minimum of 10 because its predicted worst scenarios differ from the realized tail. On MiniZinc, the selected rows cost 296 versus 217. The hindsight set can leave 23--31 cells of each selected row unevaluated when their realized losses sit at or near the public upper endpoint 1. The policy lacks a general rule for safely skipping such predicted-saturated cells.

Two alternative predictor seeds and two fold assignments produce 622--635 queries on OPENML and 981--1034 on MiniZinc for the default policy. Permuting the historical predictions across scenarios or candidates raises its cost to between $1.17$ and $1.43\times\Ccert$, but never changes the certified answer. The predicted-top-$k$ row-rule variant \cas{PG-U} is less robust under these perturbations (up to about $2.1\times\Ccert$ on OPENML).

The historical model is not being claimed as a state-of-the-art algorithm selector. As a plain per-instance selector on the challenge split \citep{LindauerEtAl2019}, the same pipeline is slightly worse than the single best algorithm on OPENML. Its value here is to order queries, not to act as a final selector.

\paragraph{Independent audit.}
An independently implemented verifier reproduced both matrices, every reported evaluation count, the exact top-three sets, and $\Ccert=554/834$. It checked that 18 deterministic and 120 random transcripts first certify at their last query. Across 576 hidden-truth perturbation tests, the next query remained unchanged when unevaluated values were altered.

\paragraph{Compute and evidence scope.}
The chooser uses 0.2--0.3 CPU seconds per run, with complete certification taking about 2 seconds on one thread; fitting the predictor once takes 21--28 seconds. These are overhead measurements on replayed data. The empirical claim still consists of \emph{two matrices, one deterministic run per method}, one ES level each, unit costs and $[0,1]$ boxes. In MiniZinc, 1528 of 2200 values are exactly 0 or 1, so saturation affects certificate behavior.

\begin{table}[tbp]
\caption{Exact evaluations until certification (ES top-3) and their multiple of $\Ccert$. All methods return the exact full-information set. $\circ$: no historical side information. $\dagger$: median over 30 seeds. $\ddagger$: control run by the independent audit.}
\label{tab:public}
\small
\setlength{\tabcolsep}{3.6pt}
\begin{tabular}{@{}lrrrr@{}}
\toprule
 & \multicolumn{2}{c}{OPENML-WEKA} & \multicolumn{2}{c}{CSP-Minizinc}\\
Method & evals & $\times\Ccert$ & evals & $\times\Ccert$\\
\midrule
Full exact ($mS$) & 3150 & 5.69 & 2200 & 2.64\\
$\circ$ Uniform over all cells$^\dagger$ & 3144 & 5.68 & 2169 & 2.60\\
$\circ$ \cas{W} & 3074 & 5.55 & 1966 & 2.36\\
$\circ$ \cas{F} & 3039 & 5.49 & 2191 & 2.63\\
$\circ$ Uniform in witness$^\dagger$ & 2006 & 3.62 & 1661.5 & 1.99\\
$\circ$ Row rule, uniform order$^{\dagger\ddagger}$ & 1905.5 & 3.44 & 1641.5 & 1.97\\
$\circ$ \cas{PD} & 1872 & 3.38 & 1699 & 2.04\\
\midrule
$\circ$ \cas{PG-U}, transcript only & 792 & 1.43 & 1126 & 1.35\\
$\circ$ \cas{PG-T} & 709 & 1.28 & 1144 & 1.37\\
\cas{PG-H} & 623 & 1.12 & 1043 & 1.25\\
\cas{PG-U} & 626 & 1.13 & 989 & 1.19\\
\textbf{\caspg{}} & \textbf{630} & \textbf{1.14} & \textbf{981} & \textbf{1.18}\\
\midrule
Hindsight $\Ccert$ (offline) & 554 & 1.00 & 834 & 1.00\\
\bottomrule
\end{tabular}
\end{table}
\begin{table}[tbp]
\caption{Where queries are spent and how historical priors change cost. Top: row groups, hindsight versus \caspg{}. Bottom: two audit draws per prior; every output remains exact.}
\label{tab:gap}
\small
\setlength{\tabcolsep}{3.6pt}
\begin{tabular}{@{}lrrrr@{}}
\toprule
 & \multicolumn{2}{c}{OPENML-WEKA} & \multicolumn{2}{c}{CSP-Minizinc}\\
Row group & $\Ccert$ & \caspg{} & $\Ccert$ & \caspg{}\\
\midrule
Three selected rows & 314 & 315 & 217 & 296\\
Excluded rows (27 / 19) & 240 & 315 & 617 & 685\\
\midrule
Prior given to \caspg{} & \multicolumn{2}{c}{evals} & \multicolumn{2}{c}{evals}\\
\midrule
Other predictor seeds & \multicolumn{2}{c}{634 / 622} & \multicolumn{2}{c}{981 / 988}\\
Random fold assignments & \multicolumn{2}{c}{631 / 635} & \multicolumn{2}{c}{1034 / 1032}\\
Permuted across instances & \multicolumn{2}{c}{753 / 647} & \multicolumn{2}{c}{1177 / 1177}\\
Permuted across candidates & \multicolumn{2}{c}{738 / 743} & \multicolumn{2}{c}{1171 / 1195}\\
\bottomrule
\end{tabular}
\end{table}
\section{What the certificate layer establishes elsewhere}
\label{sec:energy}

The benchmark experiment tests how prediction orders queries inside the core model. The next two studies instead test the \emph{certificate layer} with non-predictive schedulers. They are separate sources of evidence and use different paid units; neither is an empirical evaluation of \caspg{} on a new oracle.

\subsection{Exact recourse solves in energy scheduling}

The synthetic energy problem chooses $k$ of $A$ generation schedules using ES over the worst third of $S$ service scenarios. Here one queried action--scenario outcome is a genuine exact finite-horizon battery-recourse dynamic-program solve. Without requesting that solve, the public upper bound is the no-battery objective and the lower bound is a per-step relaxation ignoring stock coupling. The non-predictive \cas{PD} policy buys queries until the same membership certificate closes.

\paragraph{Model-equivalence check.}
Before charging one unit for each query, we verified that (i) one first-time query makes one exact recourse solve; (ii) an independent DP reproduces the paid value; (iii) public bounds are derived without hidden responses; (iv) each primitive costs one solve; and (v) full revelation matches direct exact schedule selection. All five checks pass. This justifies comparing the certificate loop against $AS$ exact solves on this oracle, rather than assuming the point-query model applies automatically.

Across $A\in\{8,16,32\}$, $S\in\{32,64,128\}$, $k\in\{1,2\}$ and eight seeds (144 scale-grid cells), the certificate avoids \textbf{93.0\%} of the sum of exhaustive exact recourse solves. The median and mean cell avoid $92.9\%$ and $91.2\%$; the nine $(A,S)$ families have median avoidance of $82.8$--$96.0\%$, and the worst cell avoids $71.0\%$. Across 148 completed energy rows, including four small exact-ratio cells, there are zero decision errors. A separately attempted scale-grid exact-ratio run timed out and contributes no certificate fields. On the four small cells ($S=16$), \cas{PD} takes between 2 and 11 times $\Ccert$: large solve avoidance is not evidence of near-optimal acquisition. The $93.0\%$ measure excludes public-bound construction and total runtime.

\subsection{Stress tests: smaller margins and looser bounds}
\label{sec:stress}

Three controlled quadratic tasks (12 candidates, 64 scenarios, $k=3$, one fixed instance per regime) vary the decision margin and curvature while retaining certified loss intervals. This asks whether an acquisition policy continues to certify or falls back safely as screening becomes less informative. The policies are \cas{W}, \cas{F} and \cas{CI} (certificate-impact ordering), not \caspg{}. The paid unit is one distinct candidate--scenario primitive in the full-exact reference, so the denominators are 4480 for benign/high curvature and 4224 for small margin, not simply $mS$.

On the benign task \cas{W} and \cas{F} require 99 and 128 exact queries, versus 250 for scenario-only screening and 4480 for candidate-only or full evaluation. At a small margin their counts rise to 191 and 192 against 257 and 4224. Under high curvature \cas{F}, \cas{CI} and \cas{W} use 576, 640 and 760 of 4480 queries; one-sided baselines were not run in that regime (\cref{fig:matched-work}). \emph{The invariant is the exact answer, not a uniformly small budget}: weaker bounds can require more exact work or the full fallback, but never justify an uncertified decision.

Separate diagnostic tests in \cref{supp:protocol} assess low-rank initialization. PCA and randomized SVD \citep{HalkoEtAl2011} can help produce initial \emph{certified} intervals; the interval certificate, not PCA/SVD, settles two benign geometry cells without queries. A hand-built rank-one PCA case retains $99.990001\%$ of variance yet misses the decisive candidate (regret 5 in aggregate risk-score units).

Predictive methods also lack stopping authority. In controlled Merton regimes \citep{Merton1976}, delta, second-order, XGBoost \citep{ChenGuestrin2016} and MLP surrogates yield outputs that may be correct or incorrect but have no certificate. These tests support the separation of prediction from stopping, not a claim that a particular surrogate is uniformly inferior.
\begin{figure}[tbp]
\centering
\begin{tikzpicture}
\begin{axis}[
  casbars, /pgf/bar width=4.6pt, width=0.80\columnwidth, height=4.9cm,
  xmax=22.5, xtick={0,5,10,15}, xlabel={share of exhaustive exact evaluations (\%)},
  ymin=-0.7, ymax=10.3, ytick={0,1,2,3,4.3,5.3,6.3,7.6,8.6,9.6},
  yticklabels={Benign: \cas{W}, \cas{CI}, \cas{F}, scenario-only, Small margin: \cas{W}, \cas{F}, scenario-only, High curvature: \cas{F}, \cas{CI}, \cas{W}},
]
\addplot[fill=casred, draw=casred] coordinates {(2.21,0) [99] (2.79,1) [125] (2.86,2) [128] (4.52,4.3) [191] (4.55,5.3) [192] (12.86,7.6) [576] (14.29,8.6) [640] (16.96,9.6) [760]};
\addplot[fill=casgrey, draw=casgrey] coordinates {(5.58,3) [250] (6.08,6.3) [257]};
\end{axis}
\end{tikzpicture}
\Description{Horizontal bar chart of exact candidate--scenario evaluations before certification on three quadratic stress tasks. Benign: CAS-W 99, CAS-CI 125, CAS-F 128, scenario-only 250, of 4480. Small margin: CAS-W 191, CAS-F 192, scenario-only 257, of 4224. High curvature: CAS-F 576, CAS-CI 640, CAS-W 760, of 4480.}
\caption{Controlled quadratic stress tasks: exact evaluations before certification (bar labels) as a share of the exhaustive denominator (4480; 4224 for small margin). Certified cost grows from 2\% to 17\% as the margin shrinks and curvature loosens the bounds, and the returned set stays exact. Candidate-only and full evaluation use the full reference budget.}
\label{fig:matched-work}
\end{figure}
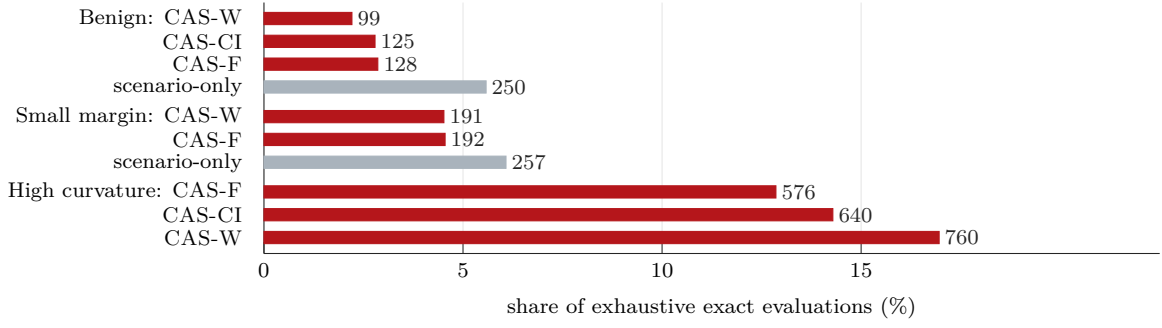
\section{Extensions and limits of the certified model}
\label{sec:spectral}

The stopping test of \cref{thm:cut} depends on valid, attained row endpoints, not on the particular tail-risk spectrum. For spectra other than ES, however, finding the cheapest query set that tightens those endpoints can be harder. We state three distinct routes below to show what the certificate framework permits. The full statements, algorithms, assumptions and proofs are in \cref{supp:coarsening,supp:lower,supp:structural,supp:ptas}. \emph{These are extensions of the offline certificate theory, not performance results for \caspg{}.}

\subsection{Three routes beyond unit-cost empirical ES}

\paragraph{Certify a coarsened spectrum (\cas{SC}).}
A spectrum can be replaced by its block-average weights on $K$ contiguous ranks. Instead of assuming that this approximation preserves the choice, \cas{SC} encloses its error over the entire current candidate box. For a contiguous partition $\mathcal P$ the spectral residual $\Omega_{\mathcal P}(y)=\Risk_\lambda(y)-\Risk_{\widetilde\lambda^{\mathcal P}}(y)$ satisfies
\[
0\le\widehat\alpha_{j,\mathcal P}\le\Omega_{\mathcal P}(y)
\le\beta_{j,\mathcal P}\qquad(y\in B_j).
\]
The rational block corrections, including the strengthened lower bound, are proved in \cref{supp:coarsen-thm}. Optimal candidate-specific partitions take $O(S^2\log S+KS^2)$ rational operations after sorting, and $K=S$ returns the original spectrum exactly. The least block count that \emph{this procedure} needs to certify a decision, $K_{\rm dec}^{\sharp}$, can equal $S$; it is not an intrinsic information dimension. A sufficient current-box condition is $(w_i+w_j)h_K<\gamma_{ij}$, where $w$ is maximum coordinate width, $h_K$ the best $K$-block half-$\ell_1$ error and $\gamma_{ij}$ the exact endpoint margin.

On the computed compression population at relative width $1/2$, median $K_{\rm dec}^{\sharp}$ is 2 for geometric spectra at $S=32,64,128$ and 1 for linear spectra at those sizes and irregular spectra at $S=32,64$. With $K=2$, candidate-aware segmentation certifies $73\%$ of 200 matched cells, versus $35\%$ for uniform ranks and $26\%$ for each of two static spectrum-only rules. At $S=128$, coverage is incomplete following memory pressure ($62.5\%,37.5\%,0\%$ across families); these are conditional certificate statistics, \emph{not measured runtime acceleration}.

\paragraph{Exploit row structure (\cas{SV}).}
For any explicit rational coherent spectrum, the \emph{lower} row frontier under unit query budgets is computable in polynomial time. The upper frontier has a fixed-band $S^{O(K)}$ algorithm, which is XP in the number of spectral bands. Exact full verification is also available when each selected row's mutable intervals split into value-separated common-anchor components (an $O(S^4)$ all-budget algorithm), or when the row has $d_{\rm ep}$ distinct endpoint pairs (an $O((d_{\rm ep}+1)2^{d_{\rm ep}}S^2\log S)$ algorithm, FPT in $d_{\rm ep}$). One-cut composition then returns the \emph{exact} hindsight certificate and a witness. Small-scale independent oracle checks at $S\le8$ have zero mismatches.

These are structural algorithms, not a general complexity classification. The unrestricted coherent-spectral upper frontier, unrestricted exact verifier and FPT versus XP in spectral band count remain \OPEN; general polyhedral-risk hardness does not settle the spectral case.

\paragraph{Price queries without changing the decision (\cas{PT}).}
With information-disjoint exact point queries, explicit rational coherent spectra, fixed ID ties and strictly positive rational prices drawn from at most a fixed $c$ classes, an offline approximation scheme returns an original-coordinate certificate satisfying
\[
\Ccert\le C_{\rm alg}\le (1+\varepsilon)\Ccert
\]
for fixed rational $0<\varepsilon<1$ in bit complexity $(S+1)^{O(c/\varepsilon)}\poly(m,L)$. This is a PTAS for fixed $c$ and only XP jointly in $(c,1/\varepsilon)$: it is not FPT, EPTAS or FPTAS. Free primitives are preprocessed. The approximation concerns certificate \emph{price}; membership remains exact. At $S=1$ the problem reduces to minimum-weight-basis verification on a uniform matroid \citep{DiwanEtAl2026}. The scheme is a hindsight result, not a cost guarantee for an online predictor.

\subsection{Shared primitives and interval-valued financial oracles}
\label{sec:finance}

FRTB-inspired hedge screening asks for candidates $h_j$ that maximize the reduction in scenario-based portfolio risk $K(P)-K(P+h_j)$ \citep{Basel2019}. Our controlled financial workload is not a production Basel FRTB calculation. It differs from the core point-query model in two decisive ways: hedge packages \emph{share} leg--scenario pricing primitives, and the nonlinear Merton pricer \citep{Merton1976} returns certified \emph{intervals} rather than atomic exact losses. We therefore verify correctness directly using interval containment and an exact fail-closed fallback. We do not transfer $\Ccert$, the online ratios, or the priced certificate PTAS to this workload, and \caspg{} was not run on it.

On nine cells ($S\in\{32,64,128\}$ and 12, 24 or 48 hedge candidates), the second-order Taylor, certified mixture-truncation and outward-rounded Merton enclosure produces the same selected membership as full evaluation. Naive package--shock calls divided by distinct shared primitives are $2.80\times$, $7.60\times$ and $17.20\times$ for the three candidate counts ($9.2\times$ pooled). Reusing one cache across candidate-count cells at fixed $S$ requires exactly $3.00\times$ fewer fresh calls than independent caches. These are opportunities available to any caching evaluator, not \emph{certificate-caused} savings.

Repeated cold-cache timings of three scaled cells show no per-cell wall-clock speedup (each paired lower 95\% confidence limit is below one); one earlier small cell has a marginal $1.05\times$ ratio. A cross-cell sharing timing test is still missing.

\paragraph{Scope summary.}
The public-query efficiency ratios concern two exact, disjoint, unit-cost empirical-ES replay tasks. Energy and quadratic experiments use non-predictive query ordering; the general-risk and price results require their own assumptions; finance uses a different shared, interval-query information model. Their work-count denominators cannot be pooled (\cref{sec:evidence-units}). A prediction-error-dependent online cost guarantee and a safe policy for skipping endpoint-saturated scenarios remain open.

\subsection{Evidence units and denominators}
\label{sec:evidence-units}

The studies measure different forms of expensive work. The reported reductions below are \emph{not} combined into a single speedup.
\begin{table}[htbp]
\centering\small
\setlength{\tabcolsep}{4pt}
\caption{Paid units and claim scope across experimental studies.}\label{tab:evidence-units}
\begin{tabular}{@{}p{0.16\textwidth}p{0.28\textwidth}p{0.43\textwidth}@{}}
\toprule
\textbf{Study} & \textbf{Paid unit / denominator} & \textbf{What the evidence establishes}\\
\midrule
ASlib replay & algorithm--instance value; $mS$ and $\Ccert$ & \caspg{}, \cas{PG-T} and certified controls: exact membership and query cost\\
Native energy & exact recourse solve; $AS$ & \cas{PD}: exact schedule and solve avoidance, not runtime\\
Quadratic stress & distinct candidate--scenario primitive; full exact & \cas{W}, \cas{F}, \cas{CI}: certificate or exact fallback\\
Financial pricing & distinct leg--scenario primitive; fresh pricer calls & exact membership and shared-cache accounting; no certified latency gain\\
\bottomrule
\end{tabular}
\end{table}

\section{Related work}
\label{sec:related}

\paragraph{Paid information under uncertainty.}
Explorable uncertainty studies what to reveal when values lie in known intervals \citep{Feder2003Median,ErlebachHoffmann2015,HalldorssonDeLima2021}. Merino and Soto analyze non-adaptive universal query sets \citep{MerinoSoto2019}; Diwan et al. optimize realized minimum-weight-basis verification \citep{DiwanEtAl2026}; and D\"urr et al. study inclusion/exclusion thresholds under a different quantifier structure \citep{DurrEtAl2024}. Megow and Schl\"oter examine uncertain set selection with additive values \citep{MegowSchloeter2023}, whereas Schl\"oter's resource augmentation relaxes solution quality \citep{Schloeter2025}. Here the value of a single selectable candidate is an ordered aggregate of many hidden scenario outcomes, and membership must remain exact.

\paragraph{Algorithms with predictions.}
Learning-augmented algorithms seek cost or performance guarantees that improve with good predictions while remaining robust to poor ones \citep{LykourisVassilvitskii2021,MitzenmacherVassilvitskii2022}. Prediction-guided query policies have also been studied for uncertain minimum spanning trees \citep{ErlebachEtAl2022}. Our invariant is zero-error membership for every prediction; the measured ratio to a computable realized hindsight certificate is not a consistency--robustness bound in prediction error.

\paragraph{Algorithm identification and tail risk.}
Classical algorithm selection predicts which algorithm to use on each instance \citep{Rice1976,BischlEtAl2016,LindauerEtAl2019}. We instead select one set of algorithms ranked by tail loss across all instances. Stochastic best-arm and top-$k$ identification, including tail-risk objectives, yield probabilistic sampling guarantees \citep{AgrawalKoolenJuneja2021,KalyanakrishnanEtAl2012}; ES simulation likewise supports statistical confidence intervals \citep{LanNelsonStaum2010}. Our uncertainty sets are deterministic and require invariance under every feasible completion.

\paragraph{Ordered risk aggregation.}
OWA and spectral risk are established aggregation frameworks \citep{Yager1988,Acerbi2002,KasperskiZielinski2015}. Chassein et al. approximate OWA-based optimization by aggregating scenarios with quality loss \citep{ChasseinEtAl2020}; Baak et al. extend OWA to interval-supported coefficients \citep{BaakEtAl2024}. Our coarsening extension retains a certified residual interval for the \emph{original} risk rule; more centrally, CAS combines query costs, ordered risk, exact membership certificates and predicted acquisition in one agent.

\section{Conclusion}
\label{sec:conclusion}

A costly selection task need not determine every candidate's risk score to know its exact best set. Under the specified disjoint exact-query model, one cut between selected upper and excluded lower keys certifies top-$k$ membership; for empirical ES, the least certifying information in hindsight can itself be computed exactly. \caspg{} uses prediction to approach that information benchmark while leaving the certificate in full control of stopping. Its two public-matrix results and the separate energy recourse study show why this division is practically meaningful, while the spectral, priced and financial extensions delineate where distinct proofs or information models are needed. The general principle is simple: \emph{predict what to evaluate, but certify what to choose}.

\clearpage
\appendix
\renewcommand{\thefigure}{S\arabic{figure}}
\renewcommand{\theHfigure}{supp.\arabic{figure}}
\setcounter{figure}{0}
\numberwithin{theorem}{section}
\section*{Extended proofs, protocols, and reproducibility evidence}
The appendices are integral to this manuscript. The core decision certificate, empirical-ES computation,
and prediction-independent exactness are supported by the proofs below; the remaining appendices
specify the boundaries of the spectral, priced, native-oracle, and nonlinear-pricing results.
The displays labeled (S.1)--(S.46) retain the original supplement numbering for traceability.

\paragraph{Purpose and scope.}
These appendices give complete mathematical arguments, protocol details, and
reproducibility specifications behind the results in the main text. They do not
add new claims beyond the assumptions, complexity classifications, negative
boundaries, and empirical denominators stated above.  \Cref{supp:cut,supp:es} prove the certificate
results that \caspg{} relies on.  \Cref{supp:caspg} proves that the answer of
\caspg{} does not depend on its predictor and specifies the public-benchmark
protocol, its ablations and its independent audit.
\Cref{supp:coarsening,supp:lower,supp:structural,supp:ptas} prove the extensions
to general spectral risk and priced certificates.  \Cref{supp:online} proves
the online boundary.  \Cref{supp:protocol,supp:merton} specify the energy,
quadratic, compression and finance studies.

\paragraph{Branch names.}
The online agents share the stopping certificate and differ in their acquisition rules (\cref{fig:model}). The theory extensions below instead modify the risk or cost model; they should not be interpreted as extra CAS-PG evaluations. Agents:
\caspg{} (prediction-guided; history and transcript), \cas{PG-T} (transcript
only, no history), \cas{PG-H} (history only), \cas{PG-U} (predicted-top-$k$
row rule), and the non-predictive \cas{PD} (primal--dual), \cas{W} (width),
\cas{F} (frontier) and \cas{CI} (certificate impact).  Theory branches:
\cas{SC} (spectral coarsening, \cref{supp:coarsening}), \cas{SV} (structural
verification, \cref{supp:lower,supp:structural}) and \cas{PT} (priced
certificates, \cref{supp:ptas}).

\section{Exact-query model and one-cut composition}
\label{supp:cut}

There are $m$ candidates and $S$ scenario coordinates per candidate.  A
realized row $x_j$ belongs to a closed Cartesian box
$B_j=\prod_{s=1}^S[\ell_{js},u_{js}]$.  A query to coordinate $(j,s)$ reveals
$x_{js}$ exactly and replaces its interval by the singleton $\{x_{js}\}$.
Rows are information-disjoint in the theorem model.  The original spectrum is
explicit, rational, and coherent,
\[
 \rho_\lambda(y)=\sum_{r=1}^{S}\lambda_r y_{(r)},\qquad
 \lambda_1\ge\cdots\ge\lambda_S\ge0,\qquad \sum_r\lambda_r=1,
 \tag{S.1}
\]
where order statistics are descending.  Distinct fixed IDs resolve numerical
ties.

For a row query set $Q$, define residual endpoint rows
\[
 u_s^Q=\begin{cases}x_s&s\in Q,\\u_s&s\notin Q,\end{cases}
 \qquad
 \ell_s^Q=\begin{cases}x_s&s\in Q,\\\ell_s&s\notin Q.\end{cases}
 \tag{S.2}
\]
Coordinatewise monotonicity of every coherent spectrum gives exact attained
row extrema
\[
 U(Q)=\rho_\lambda(u^Q),\qquad L(Q)=\rho_\lambda(\ell^Q).
 \tag{S.3}
\]
Because the box is closed, the two corner rows in (S.2) themselves
attain these extrema; no separate compactness assumption is needed.

\begin{theorem}[One-cut characterization]
\label{supp:cut-thm}
Let $A=D_k(x)$ be the realized ID-resolved top-$k$ membership set.  A row-wise
query allocation $(Q_j)_j$ certifies $A$ if and only if
\[
 \max_{i\in A}(U_i(Q_i),\mathrm{id}_i)
 <_{\rm lex}
 \min_{j\notin A}(L_j(Q_j),\mathrm{id}_j).
 \tag{S.4}
\]
For unit costs, exact row frontiers therefore compose through a finite scan of
selected-row upper frontier keys.
\end{theorem}

\begin{proof}
If (S.4) holds, every compatible selected key is no larger than its
row upper key and every excluded key is no smaller than its row lower key, so
selected and excluded membership cannot cross.

Conversely, if (S.4) fails, there is a selected $i$ and excluded $j$
whose upper/lower keys are reversed or whose equality is owned by the
excluded ID.  Cartesian row independence permits simultaneous realization of
row $i$ at its upper corner and row $j$ at its lower corner.  At that
completion $j$ precedes $i$, contradicting invariant top-$k$ membership.

For composition, define
\[
 U_j^\star(q)=\min_{|Q|\le q}U_j(Q),\qquad
 L_j^\star(q)=\max_{|Q|\le q}L_j(Q).
\]
Replace each row of any certificate by an attaining same-budget frontier set;
all separating inequalities weakly improve.  The largest selected upper key
is one of the $k(S+1)$ selected-row frontier keys.  Scanning such a key $d$
and summing, row by row, the least budget placing selected upper keys at most
$d$ and excluded lower keys strictly above $d$ yields the optimum.  The
frontier witnesses concatenate into an original-coordinate certificate.
\end{proof}

\section{Exact empirical-ES row frontiers}
\label{supp:es}

Let the empirical ES tail mass be $\tau\in(0,S]$.  The finite-sample
variational identity is
\[
 \ES_\tau(y)=\min_{z\in\R}\left[z+\frac1\tau\sum_{s=1}^S(y_s-z)_+\right].
 \tag{S.5}
\]
For integer $\tau$ this is the mean of the largest $\tau$ coordinates; for
noninteger $\tau$ the boundary observation receives the usual fractional
mass.

\subsection{Upper frontier}
For a fixed threshold $z$, revealing upper coordinate $s$ reduces the hinge
objective by
\[
 g_s(z)=(u_s-z)_+-(x_s-z)_+\ge0.
 \tag{S.6}
\]
Thus, for fixed $z$ and budget $q$, the optimal query set is the $q$ largest
$g_s(z)$.  Each gain is a plateau below $x_s$, a ramp of slope $-1$ on
$[x_s,u_s]$, and zero above $u_s$.  Two ramp pieces have equal slope and
cannot exchange order in their interiors.  Hence gain ordering changes only
at endpoints or plateau--ramp intersections.  There are $O(S^2)$ such
rational breakpoints.

Between consecutive breakpoints the set of $q$ largest gains is fixed and the
ES objective is affine in $z$, so a minimum occurs at a breakpoint.  Exact
rational evaluation of every breakpoint therefore yields $U^\star(q)$ and an
attaining query set for every $q$.  Ties are resolved with deterministic
coordinate IDs and do not change the risk value.

\subsection{Lower frontier}
For integer tail size $r$, maximizing the ES lower endpoint is equivalent to
maximizing the sum of the $r$ largest values in the vector whose coordinate
$s$ is either $\ell_s$ or $x_s$.  Sort the endpoint events and use a dynamic
program whose state records the number of coordinates processed, the number
of tail slots filled, and the number of queries used.  At coordinate $s$ the
three feasible transitions are: omit it from the tail; use $\ell_s$ in the
tail; or query and use $x_s$ in the tail.  A standard endpoint-order DP (or,
equivalently, a rank-selection DP that enforces the chosen tail entries are
the largest entries) gives the exact maximum for each budget.  Revelation
changes which coordinates occupy the tail, so this calculation does not
freeze the initially active tail.

If $\tau=r+\delta$ with $0<\delta<1$, enumerate the coordinate occupying the
fractional $(r+1)$st tail position and whether it is queried.  The remaining
$r$ full-weight positions are solved by the integer DP subject to the pivot
threshold.  There are $O(S)$ pivots and polynomially many states per pivot.
All quantities are rational combinations of input endpoints; numerator and
denominator bit lengths remain polynomial.  Consequently both ES frontiers,
attaining query sets, and their one-cut composition have polynomial bit
complexity.

\section{\texorpdfstring{\caspg}{CAS-PG}: exactness, protocol and audit}
\label{supp:caspg}

\subsection{Prediction-independent exactness}
Work in the model of \cref{supp:cut} with empirical-ES rows and unit costs.
At each iteration \caspg{} holds, for every row, the current box (public
bounds with evaluated cells replaced by their exact values) and the exact
rational endpoints $L_j,U_j$ of (S.3).  Let $A$ be the set of the $k$ smallest
keys $(L_j,\id_j)$, $i^\star=\arg\max_{i\in A}(U_i,\id_i)$ and
$j^\star=\arg\min_{j\notin A}(L_j,\id_j)$.  The run stops and returns $A$ when
$(U_{i^\star},\id_{i^\star})<_{\rm lex}(L_{j^\star},\id_{j^\star})$.  Otherwise it
evaluates one unevaluated cell of row $i^\star$ if there is one, else of row
$j^\star$; which cell is decided by a score that may depend on an arbitrary
prior $p^0$ and on the transcript.

\begin{proposition}[Prediction-independent exactness]
\label{supp:sound}
For every realization $x$ in the public box and every prior $p^0$, \caspg{}
\textup{(i)} terminates after at most $mS$ evaluations, \textup{(ii)} returns
$D_k(x)$, and \textup{(iii)} stops at the first prefix of its own query
sequence that certifies the decision.  Hence
$\Ccert(x)\le C_{\caspg}(x)\le mS$.
\end{proposition}

\begin{proof}
\emph{Invariant.}  A box coordinate changes only when the oracle returns its
exact value, so $x$ lies in the current box at every iteration, whatever $p^0$
is.  The prior enters only the choice among unevaluated cells of an already
chosen row.

\emph{Soundness of the stop.}  Suppose the test holds for $A$.  For every
completion $v$ of the current box and every $i\in A$, $j\notin A$,
$(\rho(v_i),\id_i)\le(U_i,\id_i)\le(U_{i^\star},\id_{i^\star})
<(L_{j^\star},\id_{j^\star})\le(L_j,\id_j)\le(\rho(v_j),\id_j)$.  So
$D_k(v)=A$ for every completion, in particular $A=D_k(x)$.

\emph{No late stop.}  Suppose the current transcript certifies $D_k(x)$.
By \cref{supp:cut-thm}, (S.4) holds with $A'=D_k(x)$.  Every selected lower
key is at most the corresponding upper key, which precedes every excluded
lower key; hence the $k$ smallest lower keys are exactly $A'$, so $A=A'$ and
the test fires at this iteration.  The test is evaluated before every query,
so the run stops at the first certifying prefix.  This gives (ii) and (iii).

\emph{Progress.}  Suppose the test fails and both rows $i^\star$ and $j^\star$
are fully evaluated.  Then $U_{i^\star}=L_{i^\star}$ and
$(L_{i^\star},\id_{i^\star})<(L_{j^\star},\id_{j^\star})$ because
$i^\star\in A$, $j^\star\notin A$ and IDs are distinct, so the test holds,
a contradiction.  Therefore a failed test always leaves an unevaluated cell in
row $i^\star$ or $j^\star$; each iteration evaluates a new cell, and there are
$mS$ cells.  This is (i).  The lower bound $\Ccert\le C_{\caspg}$ is the
definition of $\Ccert$ applied to the certifying set of evaluated cells.
\end{proof}

The proposition bounds nothing about cost beyond $[\Ccert,mS]$; by
\cref{supp:online} no zero-error online policy has a worst-case ratio.  The
same argument applies to every chooser plugged into the same loop (the random
baselines and \cas{PG-U} included): the loop, not the chooser, owns the
certificate.

\subsection{Score and predictor}
Inside the chosen row $r$ the score of an unevaluated cell $s$ is the exact
one-step change of the row's lower ES endpoint when the cell is pseudo-revealed
at its clipped predicted value $\hat p_{rs}$; ties are broken by the larger
prediction, then by primitive ID.  The pseudo-value is discarded.  On a
$[0,1]$ box this is the order of decreasing predicted loss.  The one-step
change of the upper endpoint is not used: on such a box it is zero until fewer
than $t$ cells of the row remain.

The prediction is $\hat p_{js}=\operatorname{clip}(p^0_{js}+e_s)$ with
$e_s=\sum_{j'\in R_s}(x_{j's}-p^0_{j's})/(|R_s|+1)$, where $R_s$ is the set of
rows already evaluated in scenario $s$.  The shrinkage constant $1$ is fixed.
The \emph{historical} prior $p^0$ is cross-fitted by scenario over the ten
folds published in the benchmark's cross-validation file (repetition 1).
Inputs are the public instance meta-features (median-imputed, signed-log
transformed, standardized; all fitted on training folds), with one output per
candidate.  The model is ridge regression with 13 penalties in
$[10^{-2},10^{4}]$ or a 300-tree random forest, chosen per outer fold by
five-fold inner cross-validation on that fold's training instances; one fixed
seed, one thread.  The proxy for an instance therefore never uses any outcome
on that instance.  It does use complete outcome columns of the other nine
folds, which the replay does not charge.  \cas{PG-T} (no history) replaces $p^0$ by
the interval midpoint $1/2$; \cas{PG-H} (history only) switches the correction
$e_s$ off.

\cas{PG-U} scores the unevaluated cells of every unresolved row that the
predicted completion ranks in the top $k$; when no such row has an
unevaluated cell it applies the \caspg{} rule.

\subsection{Tasks, data and validity rules}
Data: the public ASlib repository at a fixed commit, verified file by file
against recorded SHA-256 digests.  Every published value is parsed from its
decimal string into an exact rational.
\begin{center}\small
\begin{tabular}{lll}
\toprule
 & OPENML-WEKA-2017 & CSP-Minizinc-Obj-2016\\
\midrule
performance column & predictive accuracy (maximize) & objective (minimize)\\
loss & $1-$accuracy & objective, unchanged\\
candidates $\times$ scenarios & $30\times105$ & $22\times100$\\
records / duplicates / missing & 3150 / 0 / 0 & 2200 / 0 / 0\\
run status & 3150 ok & 1822 ok, 378 timeout\\
loss range & $[0,159/160]$ & $[0,1]$\\
cells equal to 0 / 1 & 141 / 0 & 923 / 605\\
ES level, tail mass & 0.10, $21/2$ & 0.50, 50\\
$k$, cost & 3, one per cell & 3, one per cell\\
acceptance threshold & $C\le1.50\,\Ccert$ & $C\le1.75\,\Ccert$\\
\bottomrule
\end{tabular}
\end{center}
All candidates and scenarios are retained.  The MiniZinc scenario documents
its objective as normalized to $[0,1]$ with timeouts recorded as 1, and all
378 timeout cells carry the value 1; the published value is used as the loss.
The public interval of every unevaluated cell is $[0,1]$.  Ties are owned by
the lexicographic order of stable candidate IDs.  The full-information top-3
have ES $0.5773,0.5895,0.6011$ (fourth: $0.6270$) on OPENML and
$0.5651,0.6019,0.6299$ (fourth: $0.6526$) on MiniZinc; there is no tie at
either cut.

A run is valid only if the returned set equals the independently computed
full-information set, the final transcript re-certifies, the transcript one
query shorter does not certify, and no primitive is charged twice.  All
reported runs are valid, with zero certificate violations.

\subsection{Complete results}
\begin{center}\small
\begin{tabular}{lrrrrrr}
\toprule
 & \multicolumn{3}{c}{OPENML-WEKA-2017} & \multicolumn{3}{c}{CSP-Minizinc-Obj-2016}\\
method & evals & $/mS$ & $/\Ccert$ & evals & $/mS$ & $/\Ccert$\\
\midrule
full exact $mS$ & 3150 & 1.000 & 5.686 & 2200 & 1.000 & 2.638\\
hindsight $\Ccert$ (offline) & 554 & 0.176 & 1.000 & 834 & 0.379 & 1.000\\
\cas{W} (width) & 3074 & 0.976 & 5.549 & 1966 & 0.894 & 2.357\\
\cas{F} (frontier) & 3039 & 0.965 & 5.486 & 2191 & 0.996 & 2.627\\
\cas{PD} (primal--dual) & 1872 & 0.594 & 3.379 & 1699 & 0.772 & 2.037\\
uniform over all cells (median) & 3144 & 0.998 & 5.675 & 2169 & 0.986 & 2.601\\
uniform in witness (median) & 2006 & 0.637 & 3.621 & 1661.5 & 0.755 & 1.992\\
\caspg{} & 630 & 0.200 & 1.137 & 981 & 0.446 & 1.176\\
\cas{PG-U} & 626 & 0.199 & 1.130 & 989 & 0.450 & 1.186\\
\cas{PG-T} (no history) & 709 & 0.225 & 1.280 & 1144 & 0.520 & 1.372\\
\cas{PG-U}, transcript only & 792 & 0.251 & 1.430 & 1126 & 0.512 & 1.350\\
\cas{PG-H} (history only) & 623 & 0.198 & 1.125 & 1043 & 0.474 & 1.251\\
\cas{PG-U}, history only & 623 & 0.198 & 1.125 & 1047 & 0.476 & 1.255\\
\bottomrule
\end{tabular}
\end{center}
Random baselines use seeds 0--29.  Interquartile ranges and ranges: random,
all cells 3139.5--3145.75 (3118--3150) and 2153.5--2180.5 (2134--2195);
random in witness 1975--2036.25 (1910--2132) and 1650.25--1668.75
(1615--1698).  The no-history and history-only rows were declared in the
protocol as diagnostics of the two prediction inputs, not as additional
variants for the acceptance thresholds.  The reported numbers are those of the
first execution; a second execution produced byte-identical transcripts.

\paragraph{Where the gap to $\Ccert$ lies (evaluations by row group).}
\begin{center}
\begin{tabular}{lrrr}
\toprule
 & hindsight & \caspg{} & \cas{PG-U}\\
\midrule
OPENML, three selected rows & 314 & 315 & 315\\
OPENML, 27 excluded rows & 240 & 315 & 311\\
MiniZinc, three selected rows & 217 & 296 & 300\\
MiniZinc, 19 excluded rows & 617 & 685 & 689\\
\bottomrule
\end{tabular}
\end{center}
On OPENML one near-cut row (seventh by ES, $0.634$ against the cut at
$0.601$) takes 69 (\caspg{}) or 58 (\cas{PG-U}) evaluations against a
hindsight budget of 10; the median excluded row takes 10 (hindsight 7--10).
On MiniZinc hindsight leaves 23--31 cells per selected row unevaluated because
their true value is 1 or close to it; the policy evaluates 97--100 of 100.
The remaining 68--72 extra evaluations are on excluded rows, mostly the fifth
row by ES (57 or 58 against 33).

\paragraph{Predictor sensitivity and controls (run by the independent audit).}
Evaluations, with the multiple of $\Ccert$ where stated.
\begin{center}\small
\setlength{\tabcolsep}{4pt}
\begin{tabular}{lrrrr}
\toprule
 & \multicolumn{2}{c}{OPENML-WEKA-2017} & \multicolumn{2}{c}{CSP-Minizinc-Obj-2016}\\
prior & \caspg{} & \cas{PG-U} & \caspg{} & \cas{PG-U}\\
\midrule
protocol predictor & 630 & 626 & 981 & 989\\
other predictor seeds (2) & 634 / 622 & 626 / 623 & 981 / 988 & 988 / 996\\
random fold assignments (2) & 631 / 635 & 631 / 633 & 1034 / 1032 & 1055 / 1056\\
permuted across scenarios (2) & 753 / 647 & 670 / 616 & 1177 / 1177 & 1187 / 1183\\
permuted across candidates (2) & 738 / 743 & 1184 / 1179 & 1171 / 1195 & 1263 / 1428\\
row rule, random in-row order (median, 30 seeds) & 1905.5 & & 1641.5 & \\
\bottomrule
\end{tabular}
\end{center}
Under permuted priors \caspg{} stays within $1.17$--$1.43\,\Ccert$;
\cas{PG-U} reaches about $2.1\,\Ccert$ on OPENML and $1.71\,\Ccert$ on
MiniZinc.  The row rule with a uniformly random in-row order and no model
costs $3.44\,\Ccert$ and $1.97\,\Ccert$, the same as random-in-witness, so the
whole improvement is the predicted within-row order.

\paragraph{Overhead.}
Chooser CPU per full run: 0.19--0.28\,s (\caspg{}) and about 1.3\,s
(\cas{PG-U}).  Whole run including exact certification: 1.7--3.1\,s, against
5--29\,s for \cas{W}, \cas{F} and \cas{PD}, whose receipts do not separate
scheduler time; much of the difference is fewer certificate refreshes.
Cross-fitted predictor fit, once per task: 21--28\,s CPU on one thread.

\subsection{Independent audit}
A separate verifier, sharing no checker code with the implementation (own data
reader, exact ES, certificate test and separator enumeration), reported the
following.
\begin{itemize}
\item \emph{Data.}  All source files match their digests; both matrices and
their hashes were reconstructed independently, with no duplicates or missing
cells.  The fractional ES convention was checked by hand and against the
variational form on all 52 rows.
\item \emph{Transcripts.}  For 18 deterministic and 120 random transcripts:
query count equals reported cost; no primitive is charged twice; the run is
never certified before its last query and is certified after it; the state one
query earlier admits two explicit completions with different top-3 sets.
\item \emph{Hindsight optimum.}  The 554- and 834-cell sets certify, and
removing any single cell breaks the certificate.  A separate enumeration over
lexicographic separators (36 and 97 candidates) gives the same minima; the
next-best separators cost 557 and 835.  Global optimality rests on an exchange
argument that was not machine-checked beyond agreement with the exact solver.
\item \emph{Leakage.}  In 576 attacks that altered unevaluated values, the
next query was unchanged in every case, while altering one \emph{evaluated}
value does change later queries.  Flipping the outcomes of a held-out fold
leaves that fold's predictions bit-identical.  A reachability search from the
chooser object finds no reference to the hidden matrix, the oracle or the
hindsight solver.  Only exact rationals enter the certificate path.
\item \emph{Protocol hygiene.}  The record that fixes the protocol before the
first run is a same-session hash file, not a version-control commit.  Three of
its nine hashes no longer match: a source file was reformatted, and both
predictor caches were rebuilt once to test determinism (identical predictions;
only a recorded timing field differs).  The runs were executed twice with
byte-identical transcripts.  The audit could not confirm the pre-run ordering
beyond file times, nor the statement that development used synthetic matrices
only.
\item \emph{Fairness.}  The audit regards the no-history variant as the
primary comparison against the non-predictive schedulers, because the
historical predictor is uncharged side information.  The acceptance thresholds
were fixed with pilot values of $\Ccert$ in view.
\end{itemize}

\paragraph{Scope.}
Two matrices, $[0,1]$ boxes, unit costs, $k=3$, one tail level each; no
secondary cells were run.  Nothing here concerns shared primitives, interval
oracles or non-unit costs.  The historical predictor was not designed as an
algorithm selector.  Evaluated post hoc as a plain per-instance selector on
the official train/test split of the 2017 algorithm-selection challenge
\citep{LindauerEtAl2019}, the same pipeline is slightly worse than the single
best algorithm on OPENML-WEKA-2017 (remaining gap $1.11$) and leaves $0.52$ of
the gap on CSP-Minizinc-Obj-2016.  These numbers are context for the
statement that the predictor's role is query ordering; no algorithm-selection
performance claim is made.

\subsection{The primal--dual scheduler \cas{PD}}
\cas{PD} maintains, at each uncertified transcript, a
universal witness $W$ of unevaluated primitives (every certificate must
evaluate at least one of them), raises a dual weight $y_W$ until some member
is tight under $\sum_{W\ni p}y_W\le c_p$, and evaluates a tight primitive.  In
the independent exact point-query model with positive costs this gives
$L_{\rm PD}\le\Ccert\le C_{\rm alg}\le\Gamma_{\rm run}L_{\rm PD}$ with
$L_{\rm PD}=\sum_Wy_W$ and $\Gamma_{\rm run}$ the largest witness size met in
the run.  $\Gamma_{\rm run}$ is run-dependent and can be as large as the
number of primitives, so this is an instance-relative sandwich and not a
constant or logarithmic competitive ratio.

\section{Decision-preserving spectral coarsening}
\label{supp:coarsening}

Let $\mathcal P$ be a contiguous partition of the rank set.  For block
$B=(a,b]$ define
\[
 \mu_B=\frac{1}{b-a}\sum_{r=a+1}^{b}\lambda_r,
 \qquad c_r=\lambda_r-\mu_B,
 \tag{S.7}
\]
and let $\widetilde\lambda^\mathcal P_r=\mu_B$ on that block.  With
$G(r)=\sum_{h\le r}\lambda_h$, define the chord deviation
\[
 e_r=G(r)-G(a)-(r-a)\mu_B,
 \qquad a\le r\le b.
 \tag{S.8}
\]
Concavity of $G$ gives $e_r\ge0$ and $e_a=e_b=0$.

For a current candidate box, sort lower and upper endpoint rows independently,
writing $L_1\ge\cdots\ge L_S$ and $U_1\ge\cdots\ge U_S$.  Define the block
quantities
\[
 A_B=\sum_{r=a+1}^{b-1}e_r[L_r-U_{r+1}]_+,
 \tag{S.9}
\]
\[
 B_B=\sum_{r\in B}(c_r)_+U_r-
     \sum_{r\in B}(-c_r)_+L_r,
 \tag{S.10}
\]
\[
 Z_B=\sum_{r\in B}(c_r)_+L_r-
     \sum_{r\in B}(-c_r)_+U_r.
 \tag{S.11}
\]
Put $\alpha_{\mathcal P}=\sum_BA_B$,
$\beta_{\mathcal P}=\sum_BB_B$, and
$\widehat\alpha_{\mathcal P}=\sum_B\max\{A_B,Z_B\}$.

\begin{theorem}[Uniform residual certificate]
\label{supp:coarsen-thm}
For every compatible row $y$,
\[
 0\le \alpha_{\mathcal P}\le
 \widehat\alpha_{\mathcal P}\le
 \rho_\lambda(y)-\rho_{\widetilde\lambda^\mathcal P}(y)
 \le\beta_{\mathcal P}.
 \tag{S.12}
\]
Therefore
\[
 \left[
 \rho_{\widetilde\lambda^\mathcal P}(\ell)+\widehat\alpha_{\mathcal P},
 \rho_{\widetilde\lambda^\mathcal P}(u)+\beta_{\mathcal P}
 \right]
 \tag{S.13}
\]
is a valid interval for the original spectrum.
\end{theorem}

\begin{proof}
Inside a block, summation by parts gives
\[
 \sum_{r\in B}c_ry_{(r)}
 =\sum_{r=a+1}^{b-1}e_r(y_{(r)}-y_{(r+1)})\ge0.
 \tag{S.14}
\]
Sorted endpoint monotonicity gives $L_r\le y_{(r)}\le U_r$.  Hence each
descending gap in (S.14) is at least $[L_r-U_{r+1}]_+$, proving the
block lower bound $A_B$.  Substituting $L_r$ against positive $c_r$ and
$U_r$ against negative $c_r$ gives $Z_B\le\sum c_ry_{(r)}$.  Substituting
opposite endpoints gives the upper bound $B_B$.  The maximum of two valid
lower bounds remains valid; summing over blocks proves (S.12).  Adding
these residual bounds to the coarsened endpoint risks proves (S.13).
\end{proof}

The enhanced endpoint form is additive.  If $d_r=U_r-L_r$, define
\[
 E_B=\sum_{r\in B}(\mu_B-\lambda_r)_+d_r,
 \qquad
 D_B=\sum_{r=a+1}^{b-1}e_r\min\{L_r-L_{r+1},d_{r+1}\}.
 \tag{S.15}
\]
Direct subtraction yields
\[
 \mathsf U_{\mathcal P}=\rho_\lambda(u)+\sum_BE_B,
 \qquad
 \widehat{\mathsf L}_{\mathcal P}
 =\rho_\lambda(\ell)-\sum_B\min\{D_B,E_B\}.
 \tag{S.16}
\]
Thus the best upper and lower corrections over exactly $K$ blocks are obtained
by the last-block recurrence
\[
 H[k,b]=\min_{k-1\le a<b}\{H[k-1,a]+C(a,b)\},\qquad H[0,0]=0.
 \tag{S.17}
\]
Prefix sums and the single sign change of $\lambda_r-\mu_B$ evaluate every
block cost in polynomial rational time; after $O(S^2\log S)$ preprocessing,
the segmentation DP uses $O(KS^2)$ arithmetic operations.  Revelation raises
sorted lower endpoints and lowers sorted upper endpoints, so $A_B,Z_B$ can
only improve and $B_B$ can only shrink.  Fixed-partition corrected intervals
therefore tighten monotonically; optimizing over the unchanged partition
family preserves that monotonicity.

Define $K_{\rm dec}^{\exists}$ using the original lower floor $A_B$, and
$K_{\rm dec}^{\sharp}$ using $\max\{A_B,Z_B\}$.  Both are the minimum number
of blocks, over candidate-specific partitions with at most that many blocks,
needed for the exact ID-aware separator.  Because the upper correction is
unchanged and the lower correction only improves,
$K_{\rm dec}^{\sharp}\le K_{\rm dec}^{\exists}$ whenever finite.  These are
procedure-relative certified degrees, not intrinsic lower bounds on every
possible representation.

\subsection{Current-box versus frozen-root margin bounds}
For a partition define
\[
 h(\mathcal P)=\sum_{B\in\mathcal P}\sum_{r\in B}(\mu_B-\lambda_r)_+
 =\tfrac12\|\lambda-\widetilde\lambda^\mathcal P\|_1,
 \qquad h_K=\min_{|\mathcal P|\le K}h(\mathcal P).
 \tag{S.18}
\]
If current row $j$ has maximum coordinate width $w_j$, (S.15)--
(S.16) imply that its selected-side upper inflation and excluded-side
lower deflation are each at most $w_jh(\mathcal P)$.  Hence, with exact
current-box endpoint gap $\gamma_{ij}$,
\[
 (w_i+w_j)h_K<\gamma_{ij}
 \tag{S.19}
\]
is sufficient for $K$-block certification.  Equality is also sufficient when
the selected ID owns the tie.  This is the factor-one current-box statement.
If correction bounds are frozen once on a larger root box and reused inside a
query-set optimizer, the residual envelope width is at most
$2w_jh(\mathcal P)$, giving the separate factor-two sufficient transfer
condition.  The two statements use different boxes and are not interchangeable.

\section{Arbitrary-spectrum coherent lower frontier}
\label{supp:lower}

For one lower row write $a_s=\ell_s$ and $b_s=x_s\ge a_s$.  Let
$v_0<\cdots<v_M$ be all distinct endpoints.  On cell $(v_{t-1},v_t)$ let
$n_t$ be the number of baseline lower endpoints above its interior and
$A_{ts}$ indicate whether interval $[a_s,b_s]$ crosses the cell.  With
$z_s\in\{0,1\}$ denoting revelation,
\[
 \rho_\lambda(a[z\leftarrow b])
 =v_0+\sum_{t=1}^{M}(v_t-v_{t-1})
 G\!\left(n_t+\sum_sA_{ts}z_s\right),
 \tag{S.20}
\]
where $G(r)=\sum_{h\le r}\lambda_h$ is linearly interpolated between
integers.  Because $\lambda$ is nonincreasing, $G$ is concave and
nondecreasing.

Relax $z\in[0,1]^S$ with $\sum_sz_s=q$.  The hypograph of each term in
(S.20) has the exact linear description
\[
 \eta_t\le G(r)+\lambda_{r+1}
 \left(n_t+\sum_sA_{ts}z_s-r\right),\qquad r=0,\ldots,S-1.
 \tag{S.21}
\]
Thus maximizing (S.20) is a polynomial-size rational LP.  The missing
issue is integrality.

Take an LP optimum.  If two fractional variables correspond to intervals
$I_j\subseteq I_i$, move
$\delta=\min\{1-z_i,z_j\}$ mass from $z_j$ to $z_i$.  Cardinality is
preserved; every cell covered by $I_j$ is also covered by $I_i$, so every
coverage count weakly increases.  Because $G$ is nondecreasing, the objective
cannot decrease.  At least one variable becomes integral.  Repeating leaves
fractional intervals with no containments.  Such a proper interval family has
left and right endpoints in the same strict order, so the fractional intervals
active on any cell form one contiguous block.

Order the remaining fractional variables $z_1,\ldots,z_r$, let
$s_i=\sum_{h\le i}z_h$, and note that their total is integer after subtracting
the already integral cardinality.  For $\theta\in[0,1)$ set
\[
 Z_i(\theta)=\lfloor s_i+\theta\rfloor-\lfloor s_{i-1}+\theta\rfloor.
 \tag{S.22}
\]
Then $Z_i\in\{0,1\}$, total cardinality is preserved, and
$\mathbb E Z_i=z_i$ for uniform $\theta$.  For any contiguous block of
fractional variables, its rounded sum is either the floor or ceiling of its
fractional sum and has that sum as expectation.  After adding the integral
part of a cell count, the two possibilities are adjacent integers.  The
linearly interpolated $G$ is affine between them; hence
\[
 \mathbb E_\theta G\!\left(n_t+\sum_sA_{ts}Z_s(\theta)\right)
 =G\!\left(n_t+\sum_sA_{ts}z_s\right).
 \tag{S.23}
\]
Summing over cells gives expected rounded objective equal to the LP optimum.
Only $O(S)$ distinct rounded vectors occur as $\theta$ crosses fractional
prefix breakpoints, so one can enumerate them and return an integral optimum
and witness.  Rational LP solvability, containment updates, breakpoint
arithmetic, and objective comparisons all have polynomial bit length.  This
proves polynomial bit complexity for every unit-cardinality budget.  The
argument does not extend to arbitrary weighted budgets.

\section{Structural arbitrary-spectrum upper algorithms}
\label{supp:structural}

For upper endpoint queries, let $[x_s,u_s]$ be the active query-change
interval.  With endpoint cells as above, define $b_t$ as the baseline number
of upper endpoints above cell $t$ and $c_Q(t)$ as the number of queried
intervals covering it.  Then
\[
 U(Q)=v_0+\sum_t\Delta_t G(b_t-c_Q(t)).
 \tag{S.24}
\]
This follows from the layer-cake identity
$\rho_\lambda(y)=v_0+\int G(|\{s:y_s>z\}|)\,dz$.

\subsection{Common anchor and ordered blocks}
Containment exchange is safe at unit cost: replacing a selected interval by an
unselected containing interval weakly increases every coverage $c_Q(t)$ and
therefore cannot increase (S.24).  Among fixed-cardinality optima there
is consequently a down-set in the two-rank partial order given by decreasing
$u_s$ and increasing $x_s$.

Assume first that all active intervals share an anchor,
\[
 \max_sx_s\le\min_su_s.
 \tag{S.25}
\]
Let active coordinates be indexed by decreasing $u$, and let $p(i)$ be their
rank in increasing $x$.  Let fixed inactive values be
$z_1\ge\cdots\ge z_f$ and define
\[
 B_{\rm fix}=\sum_{j=1}^{f}\lambda_jz_j,
 \quad d(v)=|\{j:z_j>v\}|,
\]
\[
 C_b(v)=\lambda_{b+d(v)}v+
 \sum_{j=d(v)+1}^{f}(\lambda_{j+b}-\lambda_{j+b-1})z_j.
 \tag{S.26}
\]
If active values $v_1\ge\cdots\ge v_n$ are inserted into the fixed
background in order, telescoping rank shifts give
\[
 \rho_\lambda(v_1,\ldots,v_n,z_1,\ldots,z_f)
 =B_{\rm fix}+\sum_{b=1}^{n}C_b(v_b).
 \tag{S.27}
\]

For exact query count $q$, construct lattice paths from $(0,n)$ to $(n,0)$.
A row event processes the next $u_i$; a column event processes the next $x$
value in decreasing order.  Coordinate $i$ is queried exactly when its row
event occurs before its column event.  Down-sets correspond to such paths.
A state $D_q(i,j,a)$ records row index $i$, column index $j$, and number $a$
of completed queried coordinates.  The number of completed unqueried
coordinates is determined by $(i,j,a)$.  If a row event closes an already
passed column, charge the corresponding upper insertion $C_b(u_i)$; if a
column event closes an already passed row, charge the realized insertion
$C_{n-q+a}(x_i)$; the two commitment events have zero immediate cost.
Every endpoint is charged once.  Condition (S.25) fixes the relative
stream order between unqueried upper and queried realized values, so no hidden
path history affects future insertion costs.  There are $O(n^2(q+1))$ states
and two outgoing transitions; all budgets require $O(S^4)$ rational
arithmetic, and backpointers recover original query coordinates.

If the active intervals split into value-separated blocks and each block has a
common anchor, run the local DP at its global rank offset and convolve the
block budget curves.  Strict-positive-overlap components cannot cross block
boundaries; conversely, if every such connected component has a common anchor,
the components ordered by value form the required block partition.  Sorting
and scanning recognizes this condition in $O(S\log S)$ comparisons.

\subsection{Endpoint-pair diversity}
Group identical active pairs $(x_s,u_s)$ into $d_{\rm ep}$ types with
multiplicity $m_h$.  At exact budget $q$, let $z_h$ be the queried count in
type $h$:
\[
 0\le z_h\le m_h,\qquad \sum_hz_h=q,\qquad z_h\in\mathbb Z.
 \tag{S.28}
\]
The cell formula becomes
\[
 \mathcal U(z)=v_0+\sum_t\Delta_t
 G\!\left(b_t-\sum_hA_{th}z_h\right).
 \tag{S.29}
\]
It is concave on the continuous box-hyperplane relaxation.  Suppose an optimal
integer count vector has two partial types $a,b$.  Move along
$z_a+\theta,z_b-\theta$.  The feasible interval in $\theta$ has integer
endpoints and at either endpoint one of the two types becomes empty or full.
Concavity implies at least one endpoint has objective no larger than the
current point; since the current point is optimal, that integer endpoint is
also optimal.  Repeating leaves at most one partial type.  Therefore it
suffices to enumerate a subset of fully queried types and an optional partial
type whose count is forced by the residual budget: at most
$(d_{\rm ep}+1)2^{d_{\rm ep}}$ patterns.  Evaluating each pattern gives
all-budget work
\[
 O((d_{\rm ep}+1)2^{d_{\rm ep}}S^2\log S)
 \tag{S.30}
\]
up to polynomial rational bit factors.

Selected upper frontiers from either route and unrestricted excluded lower
frontiers from \cref{supp:lower} compose exactly by
\cref{supp:cut-thm}.  This yields the exact minimum original-box certificate
and a witness.  The parameter $d_{\rm ep}$ is endpoint-pair diversity, not
spectral band count.

\section{Few-price-class certificate-cost PTAS}
\label{supp:ptas}

Primitive prices are strictly positive rationals from $c$ distinct classes
$V_1,\ldots,V_c$, where $c$ is fixed.  We first prove a class-count menu
theorem for interval coverage and then apply it to spectral endpoint gains.

\subsection{Coverage formulation and local representatives}
Let $z_1<\cdots<z_M$ be marker points.  For an interval set $Q$ define
$c_Q(z)=|\{I\in Q:z\in I\}|$ and
\[
 F(Q)=\sum_{t=1}^{M}B_t(c_Q(z_t)),
 \tag{S.31}
\]
where each $B_t:\{0,\ldots,S\}\to\mathbb Q_{\ge0}$ is nondecreasing and
explicitly tabulated.  Upper and lower spectral endpoint improvements take
this form after endpoint-cell decomposition.

Fix one price class and an anchor interval $I=[a,b]$.  Let
$K_I$ be class intervals containing anchor point $b$ and $C_I$ those containing
the whole anchor interval.  We construct a family $\mathcal R_I$ of
$S^{O(1/\varepsilon)}$ subsets of $K_I$, each containing $C_I$, with the
following property: for every $C_I\subseteq M_I\subseteq K_I$ there is
$R_I\in\mathcal R_I$ such that
\[
 c_{R_I}(z)\ge c_{M_I}(z)\quad\forall z,
 \qquad
 |R_I\setminus C_I|\le(1+\varepsilon)|M_I\setminus C_I|.
 \tag{S.32}
\]

To see this constructively, order $V=K_I\setminus C_I$ by increasing left
endpoint.  Enumerate all subsets of size at most $H=\lceil2/\varepsilon\rceil$,
so small targets are present exactly.  For a larger target of outside-compulsory
size $p$, put $d=\lfloor\varepsilon p\rfloor\ge\varepsilon p/2$.  If fewer
than $d$ class intervals are absent from the target, taking all of $K_I$ costs
less than $p+d\le(1+\varepsilon)p$.  Otherwise retain the whole prefix through
the $d$th omitted interval.  Partition the suffix at every $d$th target
interval; in each full part retain the same number of intervals with largest
right endpoints, with an analogous final partial part.  The extra $d$ prefix
intervals compensate for the only partially started suffix part to the left
of the anchor, while largest-right-endpoint replacement cannot reduce
coverage to the right.  There are $O(1/\varepsilon)$ parts, and enumerating
$d$, the prefix, part boundaries, and final quota yields
$S^{O(1/\varepsilon)}$ candidates.  This proves (S.32) without ever
replacing a primitive by one of another price class.

\subsection{One-class paths and once-only charging}
Make a node $(I,R)$ for every anchor and representative from its family.
Connect two real nodes only when their anchor intervals are disjoint and
ordered.  A path therefore has ordered anchors.  Assign an edge from node
$v$ to node $w$ the coverage proxy $c_{R_v\cup R_w}$ on the slab between the
anchor endpoints and cardinality charge $|R_w\setminus R_v|$.

If a physical interval appears at two path nodes, it spans both anchor
endpoints and hence contains every intermediate anchor.  It belongs to every
intermediate compulsory set and representative.  Membership is consecutive:
an interval cannot disappear and later reappear.  Consequently edge charges
sum exactly to the cardinality of the union, with every physical primitive
charged once.

For an arbitrary comparison set in this price class, close it upward under
containment (cost is unchanged inside the class).  Greedily choose disjoint
minimal anchor intervals.  Every comparison interval contains at least one
retained anchor endpoint.  At each anchor choose the representative guaranteed
by (S.32).  Consecutive representatives have exactly the same
intersection as their comparison-anchor sets: an interval from the first
representative reaching the next anchor must contain it and is compulsory
there; pointwise dominance forces equality of intersection cardinalities.
Thus unions of adjacent representatives dominate comparison coverage on each
slab.  Consecutive membership also gives
\[
 \left|\bigcup_iR_i\right|-|J_h|
 =\sum_i(|R_i|-|M_i|)
 \le\varepsilon\sum_i|M_i\setminus C_i|
 \le\varepsilon|J_h|,
 \tag{S.33}
\]
because a comparison interval can be noncompulsory only at its first retained
anchor.  Hence the canonical path dominates the comparison coverage with
$(1+\varepsilon)$ class-count inflation.

\subsection{Synchronizing price classes}
The nonlinear $B_t$ depend on total coverage, so class objectives cannot be
optimized independently.  Take the synchronized product of the $c$ local path
graphs.  A product state holds one active edge per class and the exact vector
$r=(r_1,\ldots,r_c)$ of accumulated first-appearance charges.  Over the
common slab of the active edges, score
\[
 W(e_1,\ldots,e_c)=
 \sum_{z_t\text{ in slab}}B_t\!\left(\sum_{h=1}^{c}p_{e_h}(z_t)\right).
 \tag{S.34}
\]
Advance the class whose next anchor endpoint occurs first.  Auxiliary
endpoints are made distinct, so the product is acyclic.  Consecutive
membership makes $r_h$ exactly the number of physical primitives from class
$h$ in the path union; no hidden union history is needed in the state.

For an arbitrary product path, each slab proxy is coverage by a subset of its
final union, so monotonicity of $B_t$ implies product reward no greater than
the actual $F$ value of that union.  For an arbitrary comparison set $J$,
apply the canonical construction in each price class.  The synchronized
canonical product dominates total comparison coverage on every marker, and
its class counts obey
\[
 r_h\le(1+\varepsilon)|J\cap V_h|.
 \tag{S.35}
\]
For each final count vector retain the maximum-reward product path.  At the
canonical vector its retained path has reward at least the comparison reward,
and its physical union has actual $F$ at least its product reward.  Therefore
one returned menu item dominates $J$ and obeys (S.35).

There are at most $\prod_h(|V_h|+1)\le(S+1)^c$ terminal count vectors.
Each local graph has $S^{O(1/\varepsilon)}$ size, so the synchronized dynamic
program has bit complexity
\[
 (S+1)^{O(c/\varepsilon)}\poly(m,L),
 \tag{S.36}
\]
where $L$ is total rational input encoding length.

\subsection{Spectral endpoints and full certificate}
For upper endpoint queries, cell coverage removes exceedances and the exact
gain is
\[
 U(\varnothing)-U(Q)=\sum_t\Delta_t
 [G(b_t)-G(b_t-c_Q(t))],
 \tag{S.37}
\]
a sum of nonnegative nondecreasing functions of coverage.  For lower endpoint
queries,
\[
 L(Q)-L(\varnothing)=\sum_t\Delta_t
 [G(a_t+c_Q(t))-G(a_t)],
 \tag{S.38}
\]
which has the same property.  Thus the menu theorem applies independently to
selected upper and excluded lower rows.  If query class $h$ has price $p_h$,
(S.35) gives
\[
 C(R)=\sum_hp_hr_h\le(1+\varepsilon)C(J).
 \tag{S.39}
\]

Let $J^\star$ be an exact optimum certificate.  Replace every selected row by
an upper-menu item with no larger upper key and every excluded row by a
lower-menu item with no smaller lower key.  The original lexicographic
separator is preserved, while each row price grows by at most
$1+\varepsilon$.  The maximum selected replacement key is one of the scanned
menu cuts.  Therefore the algorithm considers a valid exact-decision
certificate of price at most $(1+\varepsilon)\Ccert$.  Its output is itself a
valid certificate, so
\[
 \boxed{\Ccert\le C_{\rm alg}\le(1+\varepsilon)\Ccert.}
 \tag{S.40}
\]
For fixed $c$ this is a PTAS.  The exponent depends on both $c$ and
$1/\varepsilon$, so the joint parameterization is XP, not FPT, and the scheme
is neither an EPTAS nor an FPTAS.  In the unit-cost special case, integrality
also implies $C_{\rm alg}=\Ccert$ whenever $\varepsilon\Ccert<1$; choosing
$\varepsilon$ at this scale makes the exponent depend on $\Ccert$ and gives
at most an XP-style exact parameterization, not an FPT upgrade.  Zero-price
physical primitives are outside the theorem statement; they can be revealed
for free before applying the positive-price theorem.

\section{Online versus hindsight information}
\label{supp:online}

The PTAS is a hindsight optimizer.  A one-hot closed-box family separates it
from zero-error online acquisition.  Take $n$ unresolved coordinates with
box $[0,2/\lambda_1]^n$ and compare the unknown row with a known risk-one row.
The all-zero realization and any single-hot realization have different
winners.  On a fixed hot realization, querying the hot coordinate alone
certifies, so $\Ccert=1$.

For a deterministic zero-error policy, consider its query order under all-zero
responses and put the hot coordinate at the last location.  Until that query,
the transcript remains compatible with both decisions, so the policy uses all
$n$ queries on that fixed hot realization.  For a randomized zero-error
policy, for each realization of its random bits the average rank of the hot
coordinate over the $n$ possible locations is $(n+1)/2$.  Averaging first
over locations and then randomness implies some fixed hot location has
expected query cost at least $(n+1)/2$.  This is the standard closed-interval
search obstruction; it is reported as a boundary, not as novelty of the
queryable-uncertainty model.

\section{Reproducibility and denominator specification}
\label{supp:protocol}

\subsection{Quadratic screening}
The quadratic tasks are controlled stress tests of the certificate and of the
exact fallback; they use the non-predictive schedulers and are not compared
with $\Ccert$.  The principal matched tasks use $m=12$ candidates, $S=64$
scenarios, $k=3$, ES tail mass $1.6$, initial geometry rank $2$, and
exact-query batches of $32$.  Each regime is one fixed instance.  The fixed constructions are:
\begin{center}
\begin{tabular}{lrrrrr}
\toprule
Regime & seed & curvature & decision margin & tail density & shared-leg reuse\\
\midrule
Benign & 1101 & 0.20 & 1.00 & 0.10 & 0.60\\
Small margin & 1701 & 0.20 & 0.03 & 0.25 & 0.75\\
High curvature & 1401 & 1.50 & 0.75 & 0.20 & 0.60\\
\bottomrule
\end{tabular}
\end{center}
The paid unit is one distinct candidate--scenario primitive query after
structural reuse.  The exhaustive denominators are 4480 on benign/high
curvature and 4224 on small margin: they count the distinct priced primitives
of the full-exact reference (70 per scenario on benign), not $mS$.  On the
benign task \cas{CI} uses 125 queries (\cas{W} 99,
\cas{F} 128).  Candidate-only and scenario-only were
run on benign and small-margin tasks.  High curvature compares the three
\cas{W}, \cas{F} and \cas{CI} schedulers only against the exhaustive denominator; no missing
one-sided baseline is imputed there.

\cas{W} (width) selects the unresolved row with largest ES-width
potential and executes a geometric width-reduction epoch.  The
\cas{F} (frontier) scores primitives by certified interval width,
notional influence, liquidity/tail weight, and distance to the current
$k$th frontier.  \cas{CI} (certificate impact) caps the same influence
proxy by the nearer candidate-specific accept/reject decision gap.  These
scores only order acquisitions; the interval certificate alone authorizes
acceptance or rejection.

PCA and randomized SVD construct low-rank scenario projectors. Their
explained-variance or reconstruction score is never a stopping authority; the
projector is passed to the same deterministic quadratic interval construction
and exact certificate.  For a quadratic response
$f(x)=c+g^\top x+\tfrac12x^\top Hx$, let $z=Px$ and $e=x-z$.  The exact
identity
\[
 f(z+e)-f(z)=g^\top e+z^\top He+\tfrac12e^\top He
 \tag{S.41}
\]
is enclosed by directed arithmetic, Cauchy--Schwarz, and the symmetric
row-sum bound $\|H\|_2\le\max_i\sum_j|H_{ij}|$.  Hence the projector changes
the center/radius of a deterministic interval but never changes certificate
authority.  Randomized SVD uses a fixed seed, ten oversamples, and four power
iterations.  A zero-query cell means the initial certified intervals induced
by that geometry already separate the exact decision.  The PCA
counterexample cited in the main text is a small hand-built instance (four
scenarios, two candidates, top-1) in which uncentered rank-one PCA retains
$100/100.01=99.990001\%$ of variance and the selected candidate has regret 5.

\begin{figure}[tbp]
\centering
\includegraphics[width=0.92\textwidth]{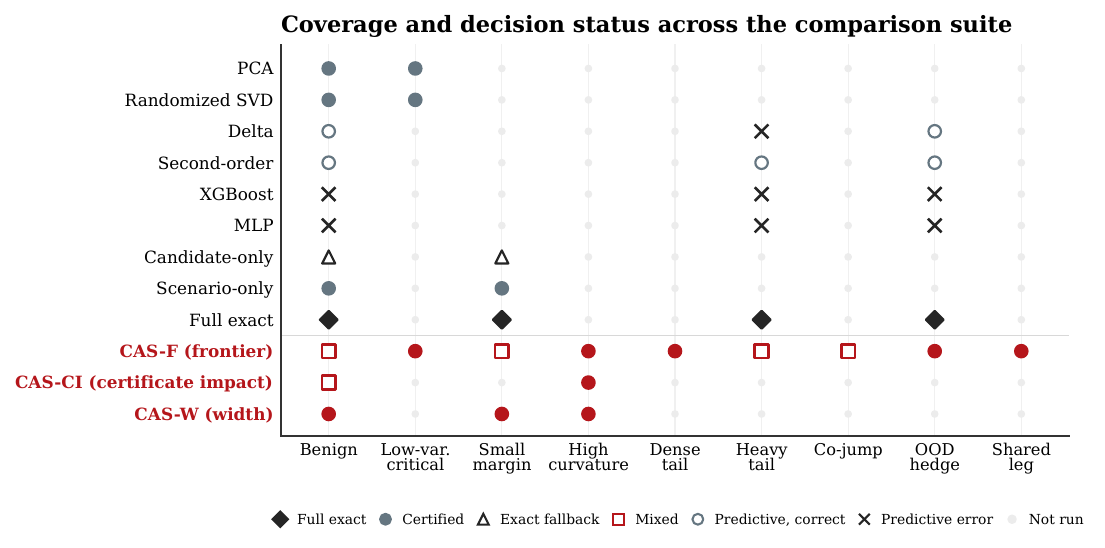}
\caption{Status map of the recorded comparison runs, grouped by method and
regime. Black diamonds: full exact evaluation; filled circles: deterministic
certificates; open triangles: exact fallback; open squares: both certification
and exact fallback occur in the same method--regime group under different
recorded settings (five cells); open circles: predictive outputs that happen
to match the reference; crosses: predictive errors; light dots: combinations
not run. Grouping does not imply every configuration has the displayed status.}
\label{supp:fig-status}
\end{figure}

\begin{figure}[tbp]
\centering
\includegraphics[width=0.70\textwidth]{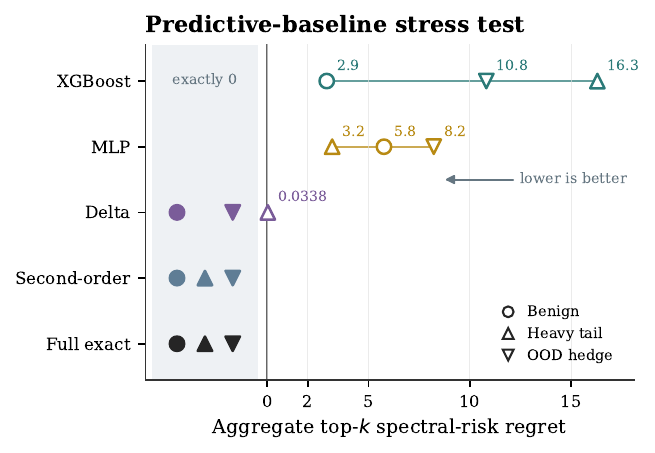}
\caption{Predictive baselines on Merton regimes: aggregate top-$k$
spectral-risk regret (S.42).  The shaded strip marks zero regret.  This is a
predictive-accuracy comparison, not an exact-work comparison; none of these
methods carries a certificate.}
\label{supp:fig-stress}
\end{figure}

\subsection{Predictive Merton baselines}
The predictive comparison uses $m=12$, $S=64$, $k=3$ and benign, Student-$t$
heavy-tail, and OOD-hedge evaluation regimes.  Two deterministic training
tasks use seed offsets 10001 and 20002 from regimes distinct from the
evaluation regime; evaluation portfolio seed, regime, and OOD hedge family
are held out.  XGBoost uses 32 trees of depth 3.  The MLP has one 16-unit
hidden layer and maximum 2000 iterations.  The predictive methods have no
certificate authority.  The plotted regret is
\[
 \sum_{j\in D_{\rm pred}}\rho_j-
 \sum_{j\in D_{\rm exact}}\rho_j,
 \tag{S.42}
\]
so it is measured in the same aggregate spectral-risk score units as the
objective.

\subsection{Spectral compression}
The recorded primary compression population has 2080 rows plus 360 positive
controls.  $K_{\rm dec}^{\sharp}$ is the smallest band count certified by the
specified enhanced correction/partition procedure, not an information-
theoretic dimension.  At $S=128$, a memory-pressure safeguard interrupted
computation: primary-world coverage is 62.5\% for the geometric family,
37.5\% for the linear/Gini-like family, and 0\% for the irregular family.
No $S=128$ irregular-family statistic is imputed.  At relative width
$r=1/2$, the observed median absolute $K_{\rm dec}^{\sharp}$ is 2 for the
geometric family at $S=32,64,128$, 1 for the linear/Gini-like family at all
three observed sizes, and 1 for the irregular family at $S=32,64$.  The
median ratio $K_{\rm dec}^{\sharp}/S$ is $1/64\approx1.6\%$ at $S=64$ and
$S=128$ and $1/32\approx3.1\%$ at $S=32$ and pooled over the computed
population.  The matched $K=2$ ablation uses the same 200 cells (a fixed
prefix of the world groups, not a random sample) for all four rules and yields certification rates $0.35,0.26,0.26,0.73$ for
uniform-rank, equal-spectral-mass, spectrum-only chord/minimax, and
candidate-aware segmentation.

\subsection{Energy}
Before using the point-query certificate comparator, the native energy adapter
passes five model-equivalence checks: (G1) one first-time action--scenario
query triggers exactly one private recourse solve and repeated queries are
memoized; (G2) the paid response equals an independently recomputed exact
finite-horizon battery DP value; (G3) public lower/upper bounds are computed
without consulting the hidden paid response; (G4) each paid primitive carries
one unit of solve cost; and (G5) full revelation under the native oracle gives
the same ID-resolved spectral-risk schedule as direct exact evaluation.  The
gate is what authorizes the clean point-query interpretation for this domain.

The scale grid has $A\in\{8,16,32\}$ actions, $S\in\{32,64,128\}$ scenarios,
$k\in\{1,2\}$, and eight seeds, giving 144 avoidance-only rows. A separately attempted scale-grid exact-ratio computation timed out and carries no computed decision/certificate fields; it is not included in the 144-row avoidance statistic. One paid
primitive is one exact action--scenario battery-recourse dynamic-program
solve, so the exhaustive denominator is
\[
 N_{\rm solve}=AS.
 \tag{S.43}
\]
Summed over the 144 cells, $93.0\%$ of these exact solves are avoided
($1-\sum Q/\sum AS=0.9304$).  The median cell avoids $92.9\%$ ($0.9287$), the
mean cell $91.2\%$, and the worst cell $71.0\%$; the nine $(A,S)$ family
medians lie between $82.8\%$ and $96.0\%$.  These figures are not total
runtime and exclude public-bound computation.  The risk functional is
empirical ES over the worst third of scenarios (tail mass $S/3$, fractional
on this grid).  The upper public bound is the no-battery cost and the lower
bound a per-step relaxation that ignores stock coupling; instances are
synthetic with horizon 4.  The acquisition rule is the non-predictive
scheduler \cas{PD} of \cref{supp:caspg}; \caspg{} was not run on this
domain.  Exact $\Ccert$ and PTAS ratios are reported only on four separate
$S=16$ cells.  On those cells \cas{PD} uses between 2 and 11
times the hindsight certificate optimum (8/4, 29/12, 11/1, 8/1), so the
avoidance result is not described as near-optimal online acquisition.

\subsection{Finance}
The finance grid crosses scenario counts $S\in\{32,64,128\}$ with candidate counts $12,24,48$. It uses shared primitives and interval-valued certified pricing,
not information-disjoint point queries.  Its $2.80\times$, $7.60\times$, and
$17.20\times$ ratios are \emph{naive per-package shock evaluations divided by
distinct shared leg--scenario primitives} for 12, 24, and 48 candidate
packages.  They measure workload sharing/cache opportunity and are not
attributed to decision certification.  The $3.00\times$ fresh-pricer-call
ratio is a separate cross-cell cache experiment comparing one cache shared
across the three candidate-count cells at fixed scenario count with three
independent caches.  Repeated cold-cell timing was run on three scaled cells
and establishes no speedup there (lower 95\% bounds of the paired ratio
$0.961$--$0.973$); a fourth, smaller anchor cell from an earlier grid shows a
marginal ratio of $1.054$ (95\% interval $1.0004$--$1.130$).  The certified path evaluates all shared primitives on this
workload, so there is no certification saving to report.  The finance result is therefore exact fail-closed
compatibility plus shared-work accounting, not a method-specific latency gain.

\section{Certified nonlinear Merton enclosure}
\label{supp:merton}

Use $\nu$ for jump intensity and $S_0$ for spot.  With log-jump mean $\mu_J$,
variance $\delta_J^2$, $m_J=\mu_J+\delta_J^2/2$, and
$\kappa=e^{m_J}-1$, define
\[
 p_n=e^{-\nu T}\frac{(\nu T)^n}{n!},\quad
 A_n=e^{-(q+\nu\kappa)T+n m_J},\quad B=Xe^{-rT},
\]
\[
 h_n=\sqrt{\sigma^2T+n\delta_J^2},\qquad
 z_n=\frac{\log(A_nS_0/B)}{h_n}+\frac{h_n}{2}.
 \tag{S.44}
\]
Merton's European call mixture is
\[
 C(S_0,\sigma)=\sum_{n\ge0}p_n
 \left[A_nS_0\Phi(z_n)-B\Phi(z_n-h_n)\right].
 \tag{S.45}
\]
On the supported compact spot/volatility shock rectangle, the certified
adapter evaluates a second-order Taylor polynomial and adds the Lagrange cubic
remainder using uniform suprema of all third mixed derivatives.  If $F(h)$ is
a mixture term after the volatility reparameterization, the third chain rule
is
\[
 C_{\sigma\sigma\sigma}
 =F_{hhh}(h')^3+3F_{hh}h'h''+F_hh'''.
 \tag{S.46}
\]
The infinite mixture is truncated with a certified tilted-Poisson tail
allowance; elementary and special-function evaluations are enclosed by
outward-rounded interval arithmetic.  These conditions imply the returned
interval contains the exact shocked price on the declared domain.  Unsupported
conditions fail closed to direct certified pricing/refinement.  This interval
oracle is deliberately not identified with the exact point-query oracle used
by the certificate PTAS.

\section{Claim boundary summary}
\begin{itemize}
\item \caspg{} is exact for every predictor (\cref{supp:sound}) and has no cost
guarantee beyond $[\Ccert,mS]$.  Its measured cost is reported on two public
tasks with one deterministic run each; the $1.14$ and $1.18$ multiples use
uncharged historical outcomes on other instances and the $1.28$ and $1.37$
multiples do not.  No algorithm-selection performance claim is made.
\item The energy, quadratic and finance studies use non-predictive
acquisition; their work reductions are properties of the certificate layer.
\item The unrestricted coherent-spectral upper frontier and unrestricted full
verifier remain open; no spectral NP-hardness or parameterized hardness result
is claimed.
\item $K_{\rm dec}^{\sharp}$ is a procedure-relative certified degree.  The
compression study demonstrates small degrees on its computed population;
it does not establish generic runtime acceleration from coarsening.
\item The few-price-class PTAS is hindsight-only, assumes information-disjoint
exact point queries, and requires fixed positive price-class count.
\item Domain experiments establish exact decisions and stated work-count
reductions under their own denominators.  They do not transfer theorem
complexity guarantees across information models.
\end{itemize}

\end{document}